\documentclass[sigconf]{acmart}

\usepackage{booktabs} 
\usepackage{psfig}
\usepackage{thmtools, thm-restate}
\usepackage{flushend}
\usepackage{setspace}
\usepackage{subfig}
\usepackage{amsmath}
\usepackage{enumitem}
\usepackage[linesnumbered,ruled,vlined]{algorithm2e}

\SetKwInOut{Parameter}{Parameter}

\newcommand*\system{\textsc{Arden}}

\setcopyright{rightsretained}

\begin{document}

\settopmatter{printacmref=false}
\renewcommand\footnotetextcopyrightpermission[1]{} 
\pagestyle{plain} 
\makeatletter
\renewcommand\@formatdoi[1]{\ignorespaces}
\makeatother

\title[Improving Performance of Private Deep Learning]{Not Just Privacy: Improving Performance of \\ Private Deep Learning in Mobile Cloud}

\author{Ji Wang}
\authornote{Both authors contribute equally to the paper.}
\affiliation{%
  \department{College of Systems Engineering}
  \institution{National University of Defense Technology}
  \city{Changsha}
  \country{P.R. China}
}
\email{wangji@nudt.edu.cn}

\author{Jianguo Zhang}
\authornotemark[1]
\affiliation{%
  \department{Department of Computer Science}
  \institution{University of Illinois at Chicago}
  \city{Chicago}
  \country{USA}
}
\email{jzhan51@uic.edu}

\author{Weidong Bao}
\affiliation{%
  \department{College of Systems Engineering}
  \institution{National University of Defense Technology}
  \city{Changsha}
  \country{P.R. China}
}
\email{wdbao@nudt.edu.cn}

\author{Xiaomin Zhu}
\affiliation{%
  \department{College of Systems Engineering}
  \department{State Key Laboratory of High Performance Computing}
  \institution{National University of Defense Technology}
  \city{Changsha}
  \country{P.R. China}
}
\email{xmzhu@nudt.edu.cn}

\author{Bokai Cao}
\affiliation{%
  \institution{Facebook Inc.}
  \city{Menlo Park}
  \country{USA}
}
\email{caobokai@fb.com}

\author{Philip S. Yu}
\affiliation{%
  \department{Department of Computer Science}
  \institution{University of Illinois at Chicago}
  \city{Chicago}
  \country{USA}
}
\affiliation{%
  \department{Institute for Data Science}
  \institution{Tsinghua University}
  \city{Beijing}
  \country{P.R. China}
}
\email{psyu@uic.edu}

\renewcommand{\shortauthors}{J. Wang et al.}

\begin{abstract}
The increasing demand for on-device deep learning services calls for a highly efficient manner to deploy deep neural networks (DNNs) on mobile devices with limited capacity. The cloud-based solution is a promising approach to enabling deep learning applications on mobile devices where the large portions of a DNN are offloaded to the cloud. However, revealing data to the cloud leads to potential privacy risk. To benefit from the cloud data center without the privacy risk, we design, evaluate, and implement a cloud-based framework \system\ which partitions the DNN across mobile devices and cloud data centers. A simple data transformation is performed on the mobile device, while the resource-hungry training and the complex inference rely on the cloud data center. To protect the sensitive information, a lightweight privacy-preserving mechanism consisting of arbitrary data nullification and random noise addition is introduced, which provides strong privacy guarantee. A rigorous privacy budget analysis is given. Nonetheless, the private perturbation to the original data inevitably has a negative impact on the performance of further inference on the cloud side. To mitigate this influence, we propose a noisy training method to enhance the cloud-side network robustness to perturbed data. Through the sophisticated design, \system\ can not only preserve privacy but also improve the inference performance. To validate the proposed \system, a series of experiments based on three image datasets and a real mobile application are conducted. The experimental results demonstrate the effectiveness of \system. Finally, we implement \system\ on a demo system to verify its practicality.
\end{abstract}

%
%


\keywords{deep learning; differential privacy; mobile cloud}

\maketitle

\section{Introduction}

Recent years have seen impressive successes of deep neural networks in various domains. DNNs demonstrate superb capabilities in discovering high-dimensional structures from large volume of data. Meanwhile, mobile devices such as smartphones, medical tools, and Internet of Things (IoT) devices have become nearly ubiquitous. There is a high demand for the on-device machine learning services, including object recognition, language translation, health monitoring, and many others \cite{He2017,Cao2017,Li2017modeling,sun2017sequential}. Encouraged by the outstanding performance of DNNs in these services, people naturally attempt to push deep learning to mobile devices \cite{Lane2016}.

The unprecedented advances of deep learning are enabled, in part, by the ability of very large deep neural networks. For example, in the ImageNet Large Scale Visual Recognition Challenge (ILSVRC) from 2010 to 2015, the continuously decreasing error rates achieved by the state-of-the-art DNNs are accompanied by their increasing complexity as shown in Fig.~\ref{fig:netsize}. There are millions to hundreds of millions of parameters in these advanced DNNs which require significant processing and storage resource. However, the miniature nature of mobile devices imposes the intrinsic capacity bottleneck that makes resource-hungry applications remain off bounds. Particularly, the large DNN models exceed the limited on-chip memory of mobile devices, and thus they have to be accommodated by the off-chip memory, which consumes significantly more energy \cite{Han2015learning, Ding2017}. What is worse, the enormous dot products aggravate the burden of processing units. Running a deep learning application can easily dominate the whole system energy consumption and drain a mobile device's battery. Up to now, it is still difficult to deploy powerful deep learning applications on mobile systems \cite{Delloitte2017}.

\begin{figure}[tb]
\centering
\includegraphics[width=3in]{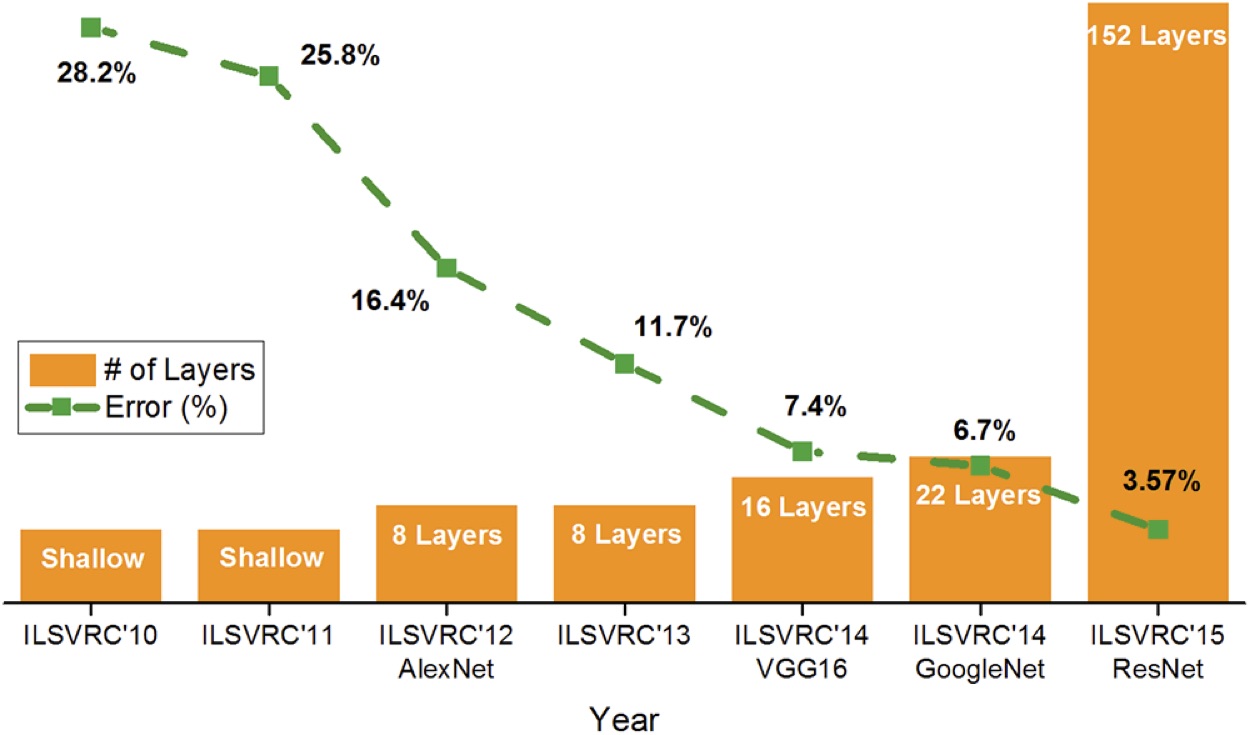}
\caption{The performance and size of the DNNs in ILSVRC'10-15 \cite{Krizhevsky2012, Simonyan2014, Szegedy2015, He2015}.} \label{fig:netsize}
\end{figure}

In order to enable on-device deep learning applications, academia and industry put forward two solutions: (1) compressing large DNNs, and (2) partitioning large DNNs across mobile devices and cloud data centers. The former one tries to prune and accelerate DNNs to fit deep learning applications into mobile devices. But this complicated technique suffers from performance degradation \cite{Ding2017} and uncontrolled energy consumption \cite{Yu2017}. For the latter solution, the shallow portions of a DNN are deployed on mobile devices while the complex and large parts are offloaded to the cloud data center. A fast transformation is first carried out on the input data locally. Then, the transformed data are revealed to the cloud data center for the time and energy consuming inferences. The experiment in \cite{Lane2015} shows that the total energy consumption can be reduced by approximately 10\% when a larger proportion of the DNN inferences are conducted in the cloud. Apart from the resource and energy consideration, the cloud-based solution is appealing for deep learning service providers. Because of the intellectual property, the deep learning service providers are not always willing to share their valuable and highly tuned DNNs \cite{Osia2017}. Deploying large parts of a DNN in the centralized data center provides a viable alternative to protect the service providers' intellectual property. The cloud-based solution paves a promising way to enable deep learning applications on mobile devices.

Despite the benefits brought by the cloud-based solution, it presents obvious privacy issues with transmitting data from mobile devices to the cloud data centers. Once an individual reveals its data to cloud data centers for further inferences, it is almost impossible for the individual to control the usage of the data. Hence, it is a critical requirement to lessen the privacy risk when implementing the cloud-based deep learning applications. Unfortunately, this is a challenging and tough problem for a number of reasons:
\begin{itemize}[leftmargin=*,noitemsep,topsep=0pt]
  \item \textbf{Guarantee the privacy}. Although the existing literature claimed that revealing the transformed data rather than the raw data could provide a satisfactory privacy protection \cite{Teerapittayanon2017, Li2017}, there still exist the threats of information disclosure \cite{Ledig2017}. Therefore, a privacy-preserving mechanism that provides a rigorous and strong privacy guarantee should be designed.
  \item \textbf{Minimize the overhead}. It is to relieve the burden of mobile devices that we resort to the cloud-based solutions. It contradicts the original purpose to cause too much overhead when providing privacy protection. Hence, we should minimize the overhead of the privacy-preserving mechanism to provide lightweight property.
  \item \textbf{Improve the performance}. In order to preserve privacy, the data released to the cloud data center is distorted and perturbed, which definitely would lead to a degradation of the inference performance, \textit{e.g.}, classification accuracy. An effective method is needed to alleviate degradation and yield a good inference result while still preserving privacy.
\end{itemize}

To address the above problems, we take both privacy and performance into consideration, and propose a priv\textbf{A}te infe\textbf{R}ence framework based on \textbf{D}eep n\textbf{E}ural \textbf{N}etworks in mobile cloud, named \system. \system\ partitions a deep neural network across the mobile device and the cloud data center. On the mobile device side\footnote{Notice that the terms mobile device side, local side, and end side are used interchangeably in this paper.}, the raw data is transformed by the shallow portions of the DNN to extract their lower-level features. The local neural networks are derived from the pretrained neural networks to avoid local training, and can be regarded as a feature extractor for different inference tasks based on the idea of transfer learning \cite{Yosinski2014}. In order to preserve privacy, we introduce the differential privacy mechanism \cite{Dwork2011} on the local side, which adds deliberate noise into the data before uploading. The large portions of the DNN are deployed in the cloud data center to run the complex and resource-hungry inference tasks. We propose a noisy training method to make the DNN robust to the additional noise in the data revealed by mobile devices, and so to improve the inference performance. Our main contributions are listed as follows:
\begin{itemize}[leftmargin=*,noitemsep,topsep=0pt]
  \item \textbf{A framework enabling deep learning on mobile devices}. We take the privacy, performance, and overhead into consideration, and design a framework to partition the very large deep neural networks in mobile cloud environment. All the resource-hungry tasks, \textit{e.g.}, network training and complex inference are offloaded to the cloud data centers. The works in the cloud data center are transparent to end mobile devices, which enables the online model upgrade.
  \item \textbf{A differentially private local transformation mechanism}. To preserve privacy, we propose a new mechanism to perturb the local data transformation based on the differential privacy mechanism. Compared with the existing ones, our proposed mechanism can be customized to protect specific data items, and fits well with the stacking structure of neural networks. The corresponding privacy budget is analyzed theoretically to provide rigorous privacy guarantee.
  \item \textbf{A noisy training approach for performance improvement}. We propose a noisy training approach that contains a generative model to inject delicate noisy samples into the training set to increase the robustness of cloud-side DNNs to the perturbed data. By this method, the negative effect incurred by the end-side noise on the inference tasks is largely alleviated.
  \item \textbf{Thorough empirical evaluation}. We evaluate the proposed framework on standard image classification tasks and a real mobile application. Further, \system\ is deployed on a real demo system to test its overhead.
\end{itemize}

\section{Preliminary}

In this section, we briefly introduce the basic knowledge of deep learning, including deep neural networks, stochastic gradient descent algorithm, and transfer learning. Then, we revisit the definition of differential privacy.

\subsection{Deep Learning}
Deep learning, which revolutionizes many machine learning tasks, transforms inputs to desired outputs by feed-forward neural networks comprising many layers of basic units like affine transformations and nonlinear activation functions. The essential idea of stacking multiple layers is to extract complex representations from high-dimensional inputs progressively. A typical neural network with two hidden layers is shown in Fig. \ref{fig:dnn}(a). Each node in the network is a neuron that takes a weighted sum of the outputs of the prior layer, and then expose the sum to the next layer through a nonlinear activation function.
\begin{figure}[htb]
\centering
\subfloat[]{\includegraphics[width=1.5in]{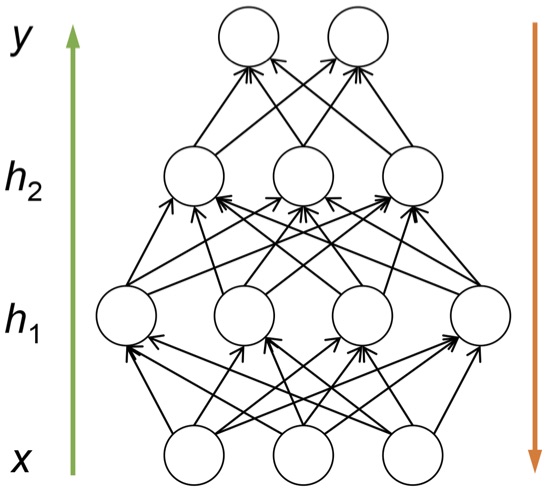}}\ \ \ \ \
\subfloat[]{\includegraphics[width=1.5in]{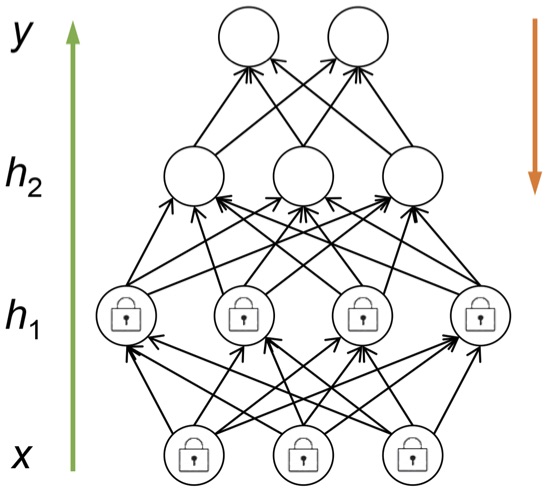}}
\caption{Illustrations of deep neural networks and transfer learning. The green and the orange arrows represent the feed-forward and the back-propagation procedures, respectively.} \label{fig:dnn}
\end{figure}

The main work of training a DNN is to automatically learn the parameters that minimize the loss function of the DNN from large volume of training data. It is usually done by the mini-batch stochastic gradient descent (SGD) algorithm and its variants \cite{Zhang2004, Hinton2006}. At each step, the SGD first calculates the loss function through the feed-forward procedure. Given the batch of random training samples $\boldsymbol{x}_N=\{x_1,...,x_{|N|}\}$ and the parameters $\boldsymbol{\omega}$, the feed-forward procedure computes the output of the DNN sequentially, and then calculates the loss for each training sample $J(\boldsymbol{\omega};x_i)$ that depicts the difference between the feedforward output and the ground truth. After that, the backpropagation procedure calculates the partial derivative of $ J(\boldsymbol{\omega};x_i)$ with respect to each parameter $\omega_j$ in $\boldsymbol{\omega}$, and then takes an average value over the batch, i.e., $g_j= \frac{{1}}{{|N|}}\sum\nolimits_i {{\nabla _{\omega_j} }J(\boldsymbol{\omega} ;{x_i})}$. The update rule for $\omega_j$ is:
\begin{equation}
{\omega _j} = {\omega _j} - \alpha {g_j},
\end{equation}
where $\alpha$ is the learning rate. One full iteration over all training samples is referred as an epoch.

As the stacking layers represent input data at a progressively higher level of abstraction, the shallow-layer representations appear to be general to diverse datasets or tasks, rather than specific to a particular one \cite{Yosinski2014}. It is plausible to regard the shallow layers of a pretrained DNN on one dataset as a generic feature extractor that can be applied to other target tasks or datasets. This is the key idea of transfer learning in deep neural networks. More precisely, the usual transfer learning approach \cite{Yosinski2014, Oquab2014} is to copy the first $n$ layers of a pretrained DNN to the first $n$ layers of a target DNN. Then, the remaining layers of the target DNN are retrained towards the task or dataset of interest while the transferred shallow layers are left frozen. One illustration of the transfer learning approach is shown in Fig. \ref{fig:dnn}(b).

\subsection{Differential Privacy}
Differential privacy is a concept of privacy tailored to the privacy-preserving data analysis. It aims at providing provable privacy guarantee for sensitive data and is increasingly regarded as a standard notion for rigorous privacy \cite{Beimel2014}. Formally, the definition of $\varepsilon$-differential privacy is given as below:

\begin{definition}~\cite{Dwork2011diff}
  A randomized mechanism $\mathcal{A}$ is $\varepsilon$-differentially private, iff for any adjacent input $d$ and $d'$, and any output $S$ of $\mathcal{A}$,
  \begin{equation}
    \Pr [\mathcal{A} (d) = S] \le {e^\varepsilon } \cdot \Pr [\mathcal{A} (d') = S].
  \end{equation}
\end{definition}

Typically, the inputs $d$ and $d'$ are adjacent inputs when they differ by only one data item. The adjacent input is an application-specific notion. For example, a sentence is divided into several items for every 5 words. Two sentences are considered to be adjacent sentences if they differ by at most 5 consecutive words. The parameter $\varepsilon$ is the privacy budget \cite{Dwork2011diff}, which controls the privacy guarantee of the randomized mechanism $\mathcal{A}$. A smaller value of $\varepsilon$ indicates a stronger privacy guarantee. According to this definition, a differentially private algorithm can provide aggregate representations about a set of data items without leaking information of any data item.

A general method for approximating a deterministic function $f$ with $\varepsilon$-differential privacy is to add noise calibrated to the global sensitivity of $f$, denoted as $\Delta f$, which is the maximal value of ${\left\| {f(d) - f(d')} \right\|}$ among any pair of $d$ and $d'$. For instance, the Laplacian mechanism is defined by,
\begin{equation}
\mathcal{A}_f (d) = f(d)+Lap(\frac{{\Delta f}}{\varepsilon }),
\end{equation}
where $Lap(\frac{{\Delta f}}{\varepsilon })$ is a random variable sampled from the Laplace distribution with scale $\frac{{\Delta f}}{\varepsilon}$.

The differential privacy enjoys a superb property, the immunity to post-processing \cite{Dwork2013}. Any algorithm cannot compromise the differentially private output and make it less differentially private. That is, if a differential privacy mechanism protects the privacy of sensitive data item, then the privacy loss cannot be increased even by the most sophisticated algorithm.

\section{The Proposed Framework}
\begin{figure}[tb]
\centering
\includegraphics[width=3.3in]{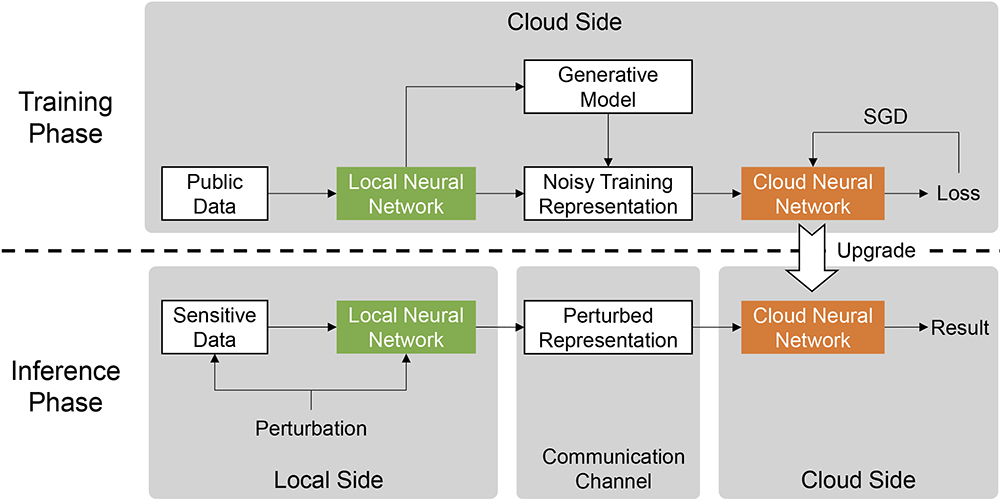}
\caption{The overview of \system.} \label{fig:framework}
\end{figure}

This section describes the proposed framework toward DNN-based private inference in mobile cloud. We first give the overview of \system. Then, the two key techniques in \system\ are detailed.

\subsection{Overview}
The overview of \system\ is presented in Fig. \ref{fig:framework}. \system\ relies on the mobile cloud environment and divides the DNN into the local-side part and the cloud-side part. The local neural network is derived from the pretrained DNN whose structure and weights are frozen. The cloud neural network is fine-tuned in the training phase. The whole training phase and the complex inference phase are performed in the cloud data centers. Mobile devices merely undertake the simple and lightweight feature extraction and perturbation.

In the inference phase, the sensitive data is transformed by the local neural network to extract the general features embedded in it. For preserving privacy, the transformation is perturbed by both nullification and random noise which are consistent with differential privacy. Then, the perturbed representations are transmitted to the cloud for further complex inference. Because the data to be transmitted are the abstract representation of the raw data, the size of the data to be transmitted is smaller than that of the raw data. The local initial transformation can reduce communication cost compared with transmitting raw data directly.

In the training phase, we use the public data of the same type as the sensitive data to train the cloud neural network. In order to improve the robustness of the cloud neural network to the noise, we propose a noisy training method where both raw training data and generative training data are fed into the network to tune the weights. The key component in noisy training is the generative model that generates sophisticated noisy training samples based on the public data. In addition, it is worth noting that the training phase and the inference phase can run in parallel once we get an initial cloud neural network. Benefitting from the transfer learning, we only need to train the cloud neural network for different datasets and tasks while keeping the local neural network frozen. It indicates that all the works on the cloud side are transparent to end mobile devices. The cloud neural network can be upgraded online without the service interruption to end users. This transparency property facilitates the deep learning service on one hand and on the other hand protects the intellectual property of service providers.

\subsection{Differentially Private Transformation}
To preserve privacy, the data transformation on the local side is perturbed. One of the key techniques in \system\ is how to inject the perturbation that satisfies the differential privacy and measure the privacy budget of the perturbation.

Regarding the local neural network as a deterministic function $\boldsymbol{x}_r = \mathcal{M}(\boldsymbol{x}_s)$, where $\boldsymbol{x}_s$ represents sensitive input data, one intuitive attempt of providing $\varepsilon$-differential privacy is to add noise that conforms to the Laplace distribution with scale $\Delta {\mathcal{M}}{\rm{/}}\varepsilon$ into the output $\boldsymbol{x}_r$. However, it is difficult to estimate the global sensitivity. An overly conservative estimation would add too much noise into the output representation, which destroys the utility of the representation for future inferences. Besides, this straightforward approach cannot be customized to provide personalized private requirements such as masking particular data items that are highly sensitive. Therefore, we propose a more sophisticated mechanism including nullification and layer-wise perturbation. The corresponding privacy budget analysis is given in detail.

Algorithm 1 outlines the differentially private data transformation on the local side. For each sensitive data $\boldsymbol{x}_s$, some data items are masked by the nullification operation. Then the data is fed into the local neural network for feature extraction. At a specific layer $l$, for the output of $\mathcal{M}_l$ (the neural network from the first layer to the $l$-th layer), we bound its infinity norm by $B$, and inject noise to protect privacy. At last, the perturbed final representation is generated by $\overline{\mathcal{M}}_l$ (the neural network after the $l$-th layer), and is transmitted to the cloud side. Next we discuss each operation in Algorithm 1 in detail and analyze the privacy budget.
\begin{algorithm}[tb]
\SetAlgoVlined
\small
\KwIn{Each sensitive data $\boldsymbol{x}_s$; Local network $\mathcal{M}(\cdot)=\overline{\mathcal{M}}_l(\mathcal{M}_l(\cdot))$.}
\Parameter{Nullification matrix $I_n$; Noise scale $\sigma$; Bound threshold $B$; Injection layer $l$.}
$\boldsymbol{x}'_s\leftarrow \boldsymbol{x}_s\odot I_n$\;
$\boldsymbol{x}_l\leftarrow \mathcal{M}_l(\boldsymbol{x}'_s)$\;
$\boldsymbol{x}'_l\leftarrow \boldsymbol{x}_l/\max(1,\frac{{{{\left\| {{\boldsymbol{x}_l}} \right\|}_{\rm{\infty}}}}}{B})$\;
$\tilde{\boldsymbol{x}}'_l\leftarrow \boldsymbol{x}'_l + Lap(B/\sigma\mathbf{I})$\;
$\tilde{\boldsymbol{x}}_r\leftarrow \overline{\mathcal{M}}_l(\tilde{\boldsymbol{x}}'_l)$\;
\KwOut{Perturbed representation $\tilde{\boldsymbol{x}}_r$.}
\small\caption{Differentially Private Transformation}
\end{algorithm}

Nullification: Given the input sensitive data $\boldsymbol{x}_s$ that consists of $N$ data items, nullification performs item-wise multiplication of $\boldsymbol{x}_s$ with $I_n$, where $I_n$ is a binary matrix constituted of 0 and 1 with the same dimensions as $\boldsymbol{x}_s$. $I_n$ can be either specified by end users to nullify the highly sensitive data items or generated randomly. The number of zeros in $I_n$ is determined by $\left\lceil {N \cdot \mu } \right\rceil$, where $\left\lceil  \cdot  \right\rceil$ is the ceiling function, and $\mu$ is the nullification rate. The zeros are located in $I_n$ conforming to the uniform distribution. Apparently, a high value of $\mu$ has a negative impact on the inference performance, which will be examined in Section 4.

Norm bounding: It is difficult to estimate the global sensitivity of neural networks. Hence, for each sensitive data $\boldsymbol{x}_s$, we clip the max value of the output of the injection layer within fixed bounds to estimate sensitivity, \textit{i.e.}, the output $\boldsymbol{x}_l$ is bounded as $\boldsymbol{x}_l/\max(1,\frac{{{{\left\| {{\boldsymbol{x}_l}} \right\|}_{\rm{\infty}}}}}{B})$. It indicates that $\boldsymbol{x}_l$ is preserved when ${\left\| {{\boldsymbol{x}_l}} \right\|}_{\rm{\infty}}\leq B$, whereas it is scaled down to $B$ when ${\left\| {{\boldsymbol{x}_l}} \right\|}_{\rm{\infty}}> B$. Then, the global sensitivity can be estimated as 2$B$. The bound threshold $B$ is input-independent and so would leak no sensitive information.  In practice, the value of $B$ can be set as the median of the infinity norm of the original outputs over the training \cite{Abadi2016}.

We add random noise sampled from the Laplace distribution into the bounded output $\boldsymbol{x}'_l$ to protect the privacy. Different from the existing works where the noise is added into the final output of the deterministic function, the noise is added during the transformation in Algorithm 1. This injection method is more flexible, and fits much better with the stacking structure of neural networks. One important issue of differentially private transformation is determining the privacy budget that represents a rigorous privacy guarantee provided by the designed mechanism. It is given as follows.
\begin{restatable}{theorem}{fta}
\label{thm:privacy}
Given the sensitive data $\boldsymbol{x}_s$ and the local neural network $\mathcal{M}$, Algorithm 1 is $\varepsilon$-differentially private,
\begin{equation}
\varepsilon  = \ln [(1-\mu) {e^{2\sigma /\Lambda}} + \mu],
\end{equation}
where $\Lambda = \|{\nabla _{{{\boldsymbol{x}}'_l}}}{\overline {\mathcal M} _l}\|_{\infty}$.
\end{restatable}

See Appendix A for the proof.

We can find that the Laplacian mechanism is a special case of Theorem 3.1. If we do not perform the nullification, and directly add noise into the final output of the local neural network, then $\mu=0$, $\overline{\mathcal{M}}_l(\boldsymbol{x})=\boldsymbol{x}$, and so ${\nabla _{{{\boldsymbol{x}}'_l}}}{\overline {\mathcal M} _l}=\mathbf{1}$. As a result, $\varepsilon = 2\sigma$, which agrees with the Laplacian mechanism.

\subsection{Noisy Training}
The differentially private transformation on the mobile side provides a guaranteed privacy for sensitive data, but perturbing the local transformation incurs the performance sacrifice in the cloud-side inferences nonetheless. The cloud-side DNN trained by the conventional training method where the loss of the DNN is minimized on the purely clean training data lacks the robustness to the noisy representations revealed by mobile devices. In order to mitigate the negative effectiveness caused by the perturbation, we propose the noisy training method, another key technique in \system, to enhance the robustness of cloud-side DNNs.

One explanation of the DNN's lack of robustness is that noisy representations are the blind points of the DNN trained by the conventional method. The robustness of cloud-side DNN could be enhanced by training on a mixture of noisy and clean representations. The new training loss is defined as:
\begin{equation}\label{eqn:loss-1}
\begin{split}
 J(\boldsymbol{\omega} ;{\boldsymbol{x}_r},{\boldsymbol{\widetilde x}_r}) & = \lambda J(\boldsymbol{\omega} ;{ \boldsymbol{x}_r}) + (1 - \lambda )J(\boldsymbol{\omega} ;{ \boldsymbol{\widetilde x}_r}), \\
 {\boldsymbol{\widetilde x}_r} & = {\boldsymbol{x}_r} + Lap(B/\sigma \mathbf{I}).
\end{split}
\end{equation}

Instead of calculating training loss merely on the clean representations, (\ref{eqn:loss-1}) combines the losses on both the clean and the noisy representations, and uses $\lambda$ to control the tradeoff between the two losses. Then, the robustness of the cloud-side DNN to noisy representations could be improved.

However, the way in which the random noise is added during the inference is unknown to anyone except for the mobile device itself. In addition, due to the nature of randomness, it is obviously impractical to train the cloud-side DNN on all possible noisy representations. A deviation from the trained noisy representations could still be the blind point of the cloud-side DNN. To overcome this problem and further enhance the robustness, we train the cloud-side DNN under the worst situation, which can be viewed as a min-max problem. We inject the worst perturbation $\boldsymbol{r}$ into the generated noisy representations to maximize the DNN's deviation from the original output, \textit{i.e.}, the maximal loss $J(\boldsymbol{\omega};\boldsymbol{\widetilde x}_r+\boldsymbol{r})$, while the DNN tries to minimize the deviation through training:
\begin{equation}\label{eqn:max-min}
\min \mathop {\max }\limits_{\boldsymbol{r};||\boldsymbol{r}|{|_2} \le \eta } J(\boldsymbol{\omega} ;{ \boldsymbol{\widetilde x}_r} + \boldsymbol{r}),
\end{equation}
where $\boldsymbol{r}$ is the perturbation added to the trained noisy representations, and $\eta$ controls the scale of the noise.

During the noisy training, we determine the worst perturbation that maximizes the deviation, and train the cloud-side DNN to be robust to such perturbation through (\ref{eqn:max-min}). Generally, it is difficult to obtain a closed form for the exact $\boldsymbol{r}$ that maximizes $J(\boldsymbol{\omega} ;{ \boldsymbol{\widetilde x}_r} + \boldsymbol{r})$, especially for complex models such as the DNN. To tackle this problem, we use the first-order Taylor series to approximate $J(\boldsymbol{\omega} ;{ \boldsymbol{\widetilde x}_r} + \boldsymbol{r})$:
\begin{equation}
 J(\boldsymbol{\omega} ;{ \boldsymbol{\widetilde x}_r} + \boldsymbol{r}) \approx J(\boldsymbol \omega ;{\boldsymbol{\widetilde x}_r}) + {\nabla _{{\boldsymbol{\widetilde x}_r}}}J(\boldsymbol \omega ;{\boldsymbol{\widetilde x}_r})\boldsymbol{r}.
\end{equation}

The perturbation to the noisy representation causes the deviation to grow by ${\nabla _{{\boldsymbol{\widetilde x}_r}}}J(\boldsymbol \omega ;{\boldsymbol{\widetilde x}_r})\boldsymbol{r}$. We can maximize this increase subject to the L2-norm constraint on $\boldsymbol{r}$ by assigning the perturbation in the gradient direction of $J(\boldsymbol \omega ;{\boldsymbol{\widetilde x}_r})$ with respect to $\boldsymbol{\widetilde x}_r$:
\begin{equation}\label{eqn:noise}
\begin{split}
& \boldsymbol{r} = \eta \frac{\boldsymbol{g}}{{||\boldsymbol{g}|{|_2}}},\\
& \boldsymbol{g} = {\nabla _{{\boldsymbol{\widetilde x}_r}}}J(\boldsymbol \omega ;{\boldsymbol{\widetilde x}_r}).
\end{split}
\end{equation}

\begin{algorithm}[tb]
\SetAlgoVlined
\small
\KwIn{Clean representation $\boldsymbol{x}_r$; Cloud-side DNN $\mathcal{C}(\boldsymbol \omega)$.}
\Parameter{Batch size $N$; Noise scale $\sigma$; Bound threshold $B$; Controllers $\lambda$ and $\eta$; Learning rate $\alpha$.}
$\boldsymbol{d}\leftarrow\mathbf{0}$\;
\ForEach{$\boldsymbol{x}^{(i)}_r$ in $\{\boldsymbol{x}^{(1)}_r,...,\boldsymbol{x}^{(N)}_r\}$}{
    $\mathcal{L}_1 \leftarrow$ Loss($\boldsymbol{y}^{(i)}_{true}; \mathcal{C}(\boldsymbol \omega; \boldsymbol{x}^{(i)}_r)$)\;
    $\boldsymbol{\widetilde x}^{(i)}_r\leftarrow {\boldsymbol{x}_r} + Lap(B/\sigma \mathbf{I})$\;
    $\mathcal{L}_2 \leftarrow$ Loss($\boldsymbol{y}^{(i)}_{true}; \mathcal{C}(\boldsymbol \omega; \boldsymbol{\widetilde x}^{(i)}_r$)\;
    $\boldsymbol g \leftarrow \nabla _{{\boldsymbol{\widetilde x}^{(i)}_r}}\mathcal{L}_2$\;
    $\boldsymbol{r} \leftarrow \eta \frac{\boldsymbol g}{\|\boldsymbol g\|_2}$ \;
    $\mathcal{L}_3 \leftarrow$ Loss($\boldsymbol{y}^{(i)}_{true}; \mathcal{C}(\boldsymbol \omega; \boldsymbol{\widetilde x}^{(i)}_r+\boldsymbol{r}$)\;
    $\mathcal{L} \leftarrow \lambda \mathcal{L}_1 + (1-\lambda)(\mathcal{L}_2 + \mathcal{L}_3)$\;
    $\boldsymbol{d}\leftarrow \boldsymbol{d} + {{\nabla _{\boldsymbol{\omega}} }\mathcal{L}}$\;
}
$\boldsymbol{\omega}\leftarrow \boldsymbol{\omega} - \alpha \frac{\boldsymbol{d}}{N}$\;
\small\caption{Noisy Training in Each Batch}
\end{algorithm}

The gradient $\boldsymbol{g}$ can be efficiently computed by the backpropagation \cite{Rumelhart1986}. Then, the training loss is defined as:
\begin{equation}\label{eqn:loss-2}
J(\boldsymbol{\omega} ;{\boldsymbol{x}_r},{\boldsymbol{\widetilde x}_r}) = \lambda J(\boldsymbol{\omega} ;{ \boldsymbol{x}_r}) + (1 -\lambda )[J(\boldsymbol{\omega} ;{ \boldsymbol{\widetilde x}_r})+ J(\boldsymbol{\omega} ;{ \boldsymbol{\widetilde x}_r}+\boldsymbol{r})].
\end{equation}

Algorithm 2 outlines the noisy training framework for the cloud-side DNN. It is designed based on the stochastic gradient descent algorithm. Different from the conventional SGD, the noisy training jointly minimizes the loss on clean representations $\mathcal{L}_1$, the loss on noisy representations $\mathcal{L}_2$, and the loss on perturbed noisy representations $\mathcal{L}_3$. In each batch, Algorithm 2 firstly computes the loss on clean representations. Then, the noisy representations are generated, and are fed into the cloud-side DNN to get the loss on noisy representations. In order to augment the robustness to the random noisy representations, Algorithm 2 computes the gradient of $\mathcal{L}_2$ with respect to the noisy representations and generates the worst perturbation to the noisy representations. Finally, the backpropagation calculates the partial derivative of the joint loss with respect to each parameter in $\boldsymbol{\omega}$. The parameters $\boldsymbol{\omega}$ are updated by an average value over the batch.

\section{Experimental Evaluation}
In this section, image classification tasks and a real mobile application are used as experimental examples to evaluate the effectiveness of \system. We first examine the effect of different parameters based on two image benchmark datasets, MNIST \cite{LeCun1998} and SVHN \cite{Netzer2011}, and then verify the performance improvement based on CIFAR-10 \cite{Krizhevsky2009} and the preceding two datasets. In addition, in order to verify the effectiveness in a real scenario, we use a mobile application\footnote{\url{http://www.biaffect.com}} DeepMood \cite{Cao2017} to test \system's performance and analyze the performance under different privacy budgets. DeepMood harnesses the sequential information collected from the basic keystroke patterns and the accelerometer on the phone (\textit{e.g.}, alphanumeric character typing pattern, special character typing pattern, and accelerometer values) to predict the user's mood.

For image classification tasks, three widely used convolutional deep neural networks (Conv-Small, Conv-Middle, and Conv-Large) are implemented in \system\ \cite{Laine2017, Park2017}. We derive the local neural network from Conv-Small, which is pretrained on CIFAR-100 dataset \cite{Krizhevsky2009}. The first 3 layers of the pretrained Conv-Small are deployed as the local neural network. For MNIST and SVHN, we use Conv-Middle as the cloud-side DNN. In the performance comparison, Conv-Large is used for CIFAR-10. For different datasets and cloud-side DNNs, the local neural network remains unchanged to show the transfer learning ability and the transparency property.

For DeepMood, the single-view DNN proposed in \cite{Cao2017} is used as the cloud-side DNN. We add an extra dense layer as the local network which is pretrained on special character typing pattern. Alphanumeric character typing pattern (ALPH.) and accelerometer values (ACCEL.) are used to test \system's performance.

The proposed models are implemented using TensorFlow \cite{Tensorflow}. For image classification tasks, the learning rate and batch size are set as 0.0015 and 128, respectively. The numbers of epochs for MNIST, SVHN, and CIFAR-10 are 35, 45, and 70, respectively. For DeepMood, the learning rate, the batch size, and the epochs are 0.001, 256, and 100, respectively. Considering the randomness during perturbation, we run the test experiments ten times independently to obtain an averaged value. The code is provided for reproducibility.

We suppose that the deep learning service provider and the mobile device users reach a consensus over the perturbation strength in advance, which is a common practice in reality. The perturbation strength is represented by $(b, \mu)$, where $b$ is the diversity of the Laplace distribution, and $\mu$ is the nullification rate. In this section, the perturbation strengths for MNIST, SVHN, CIFAR-10, and DeepMood (including ALPH. and ACCEL.) are set as (5, 10\%), (2, 5\%), (2, 5\%), and (0.5, 10\%), respectively. When the perturbation is injected to the last layer of local network, based on Theorem~\ref{thm:privacy}, the privacy budgets for MNIST, SVHN, CIFAR-10, ALPH., and ACCEL. are 0.7, 3.7, 3.5, 7.8, and 9.8, respectively.

\subsection{Parameter Selection}
The parameters $\lambda$ and $\eta$ are two important parameters in the proposed noisy training. Fig.~\ref{fig:parameter} shows the change of the performance with varying $\lambda$ and $\eta$. Due to space limitation, here we focus on MNIST and SVHN. It can be found that when $\eta$ is too small, it has limited ability to resist the randomness of perturbation. However, when $\eta$ is too large, the loss $\mathcal{L}_3$ would overwhelm the other losses, which makes the DNN focus more on resisting randomness rather than classifying samples correctly. Because a small value of $\lambda$ means a large weight of $\mathcal{L}_3$ in the joint loss, small $\lambda$ usually aggravates this phenomenon. Based on the results in Fig.~\ref{fig:parameter}, $\eta$ is set as 5 for the following experiments. The effect of $\lambda$ will be further discussed.
\begin{figure}[tb]
\centering
\subfloat[MNIST (5,10\%)]{\includegraphics[width=1.625in]{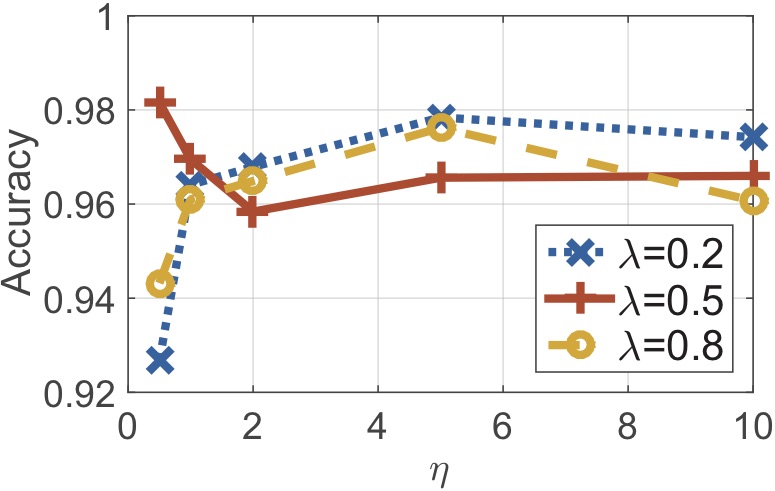}}\ \ \
\subfloat[SVHN (2,5\%)]{\includegraphics[width=1.625in]{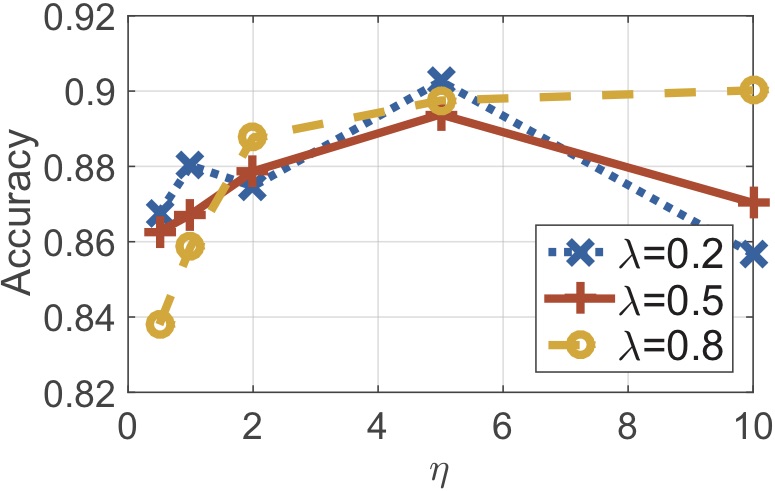}}
\caption{Effect of $\eta$ and $\lambda$.} \label{fig:parameter}
\end{figure}

\subsection{Impact of Private Perturbation}
Although it is assumed that there is a prior consensus between the mobile device users and the deep learning service provider, mobile device users may change their perturbation strength. Hence, we investigate the impact of different perturbation strength on the performance when the model has been trained with a pre-assigned perturbation strength. It can be seen from Fig.~\ref{fig:Perturb} that the accuracy usually reaches the peak around the pre-assigned perturbation strength. An interesting observation is that when the perturbation strength is 0, \textit{i.e.}, $b=0, \mu=0$, the accuracy is unexpectedly low, especially for $\lambda=0.2, 0.5$. This is because a small value of $\lambda$ makes the DNN biased to the noisy samples, which weakens the DNN's ability to classify the purely clean samples. When the perturbation strength is too large, it exceeds the resistance ability given by the noisy training, so the accuracy decreases. Owing to the noisy training, the performance changes are within 10\% in most cases.
\begin{figure}[tb]
\centering
\subfloat[Effect of $b$ on MNIST]{\includegraphics[width=1.53in]{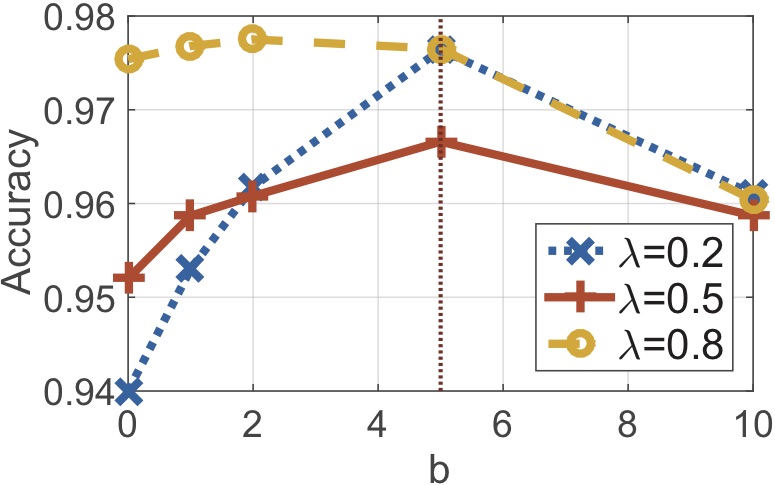}}\ \ \
\subfloat[Effect of $\mu$ on MNIST]{\includegraphics[width=1.53in]{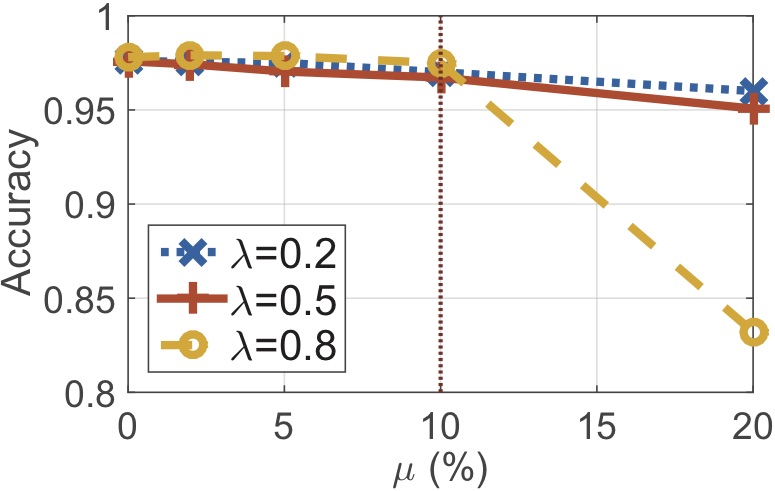}} \\
\subfloat[Effect of $b$ on SVHN]{\includegraphics[width=1.53in]{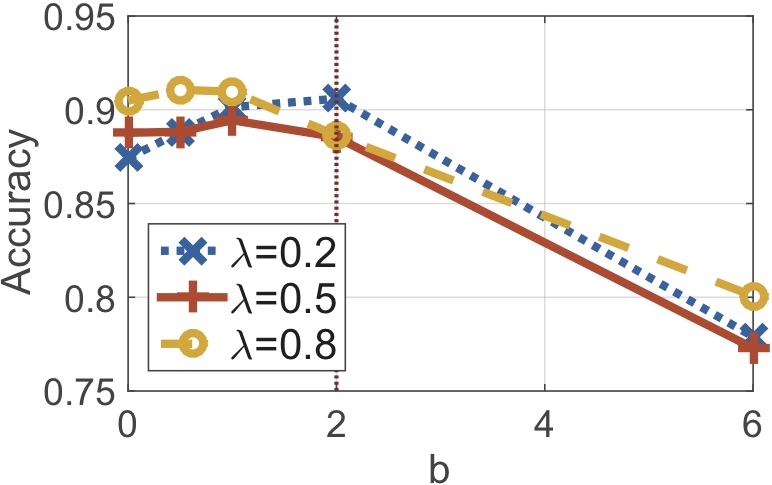}}\ \ \
\subfloat[Effect of $\mu$ on SVHN]{\includegraphics[width=1.53in]{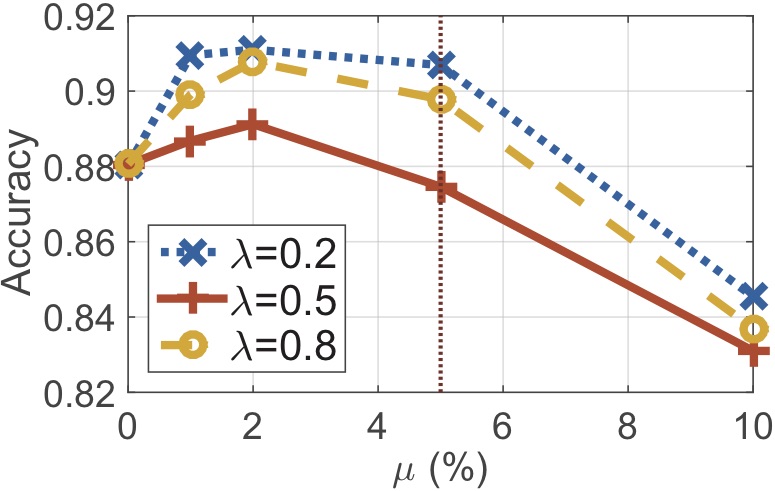}}
\caption{Effect of $b$ and $\mu$ when the model has been trained with the pre-assigned perturbation which is represented by the brown dash line.} \label{fig:Perturb}
\end{figure}

As the proposed perturbation method in Section 3 enables the layer-wise perturbation, we further study the impact of injecting perturbation into the different layers of the local neural network. The results in Table~\ref{table:layer} show that the DNN trained by the proposed noisy training method is robust to the layer-wise perturbation. For different perturbation, the DNN can keep a relatively high accuracy.

\begin{table}[tb]
  \caption{Accuracy in Layer-Wise Perturbation (\%)} \label{table:layer}
  \begin{tabular}{ccccc}
  \toprule
    & Input & Layer 1 & Layer 2 & Layer 3\\
    \midrule
    MNIST & 97.67 & 97.24 & 97.57 & 98.02\\
    SVHN& 88.56 & 88.66 & 87.82 & 88.12\\
  \bottomrule
  \end{tabular}
\end{table}
To demonstrate the effectiveness of private perturbation intuitively, we visualize noise and reconstruction in Fig.~\ref{fig:reconstruct}. The convolutional denoising autoencoder \cite{Masci2011}, which has been successfully applied to image denoising and super-resolution reconstruction \cite{Radford2015, Dong2016}, is used to reconstruct the original data from the perturbed ones. We train the model based on two perturbation strengths. When the perturbation strength is relatively weak, the perturbed pictures can to some extent be reconstructed, although the reconstructed pictures are quite different from the original ones. When the perturbation strength reaches (5,10\%), the strength used in our experimental settings, the perturbed pictures can hardly be reconstructed, which indicates that the adversarial attacker cannot restore the original data based on the revealed perturbed ones even the perturbation strength is public.

\subsection{Performance Comparison}

To verify the performance improvement brought by \system, we compare it with three variants on three image datasets and two real-scenario datasets. BASE denotes the cloud-side DNN, \textit{i.e.}, Conv-Middle for MNIST and SVHN, Conv-Large for CIFAR-10, a sigle-view DNN for DeepMood. The training data and the testing data are fed into the network directly without local transformation, which can be regarded as revealing raw data to the cloud. \system-L1 means that only the loss $\mathcal{L}_1$ is considered when training the cloud-side DNN. It is tested in two situations: test without private perturbation and test with private perturbation. \system\ is the complete framework proposed in this paper. $\lambda$ is set as 0.2 for MNIST, SVHN, and DeepMood, 0.5 for CIFAR-10.

\begin{figure}[tb]
\centering
\subfloat[(1, 1\%)]{\includegraphics[width=1.58in]{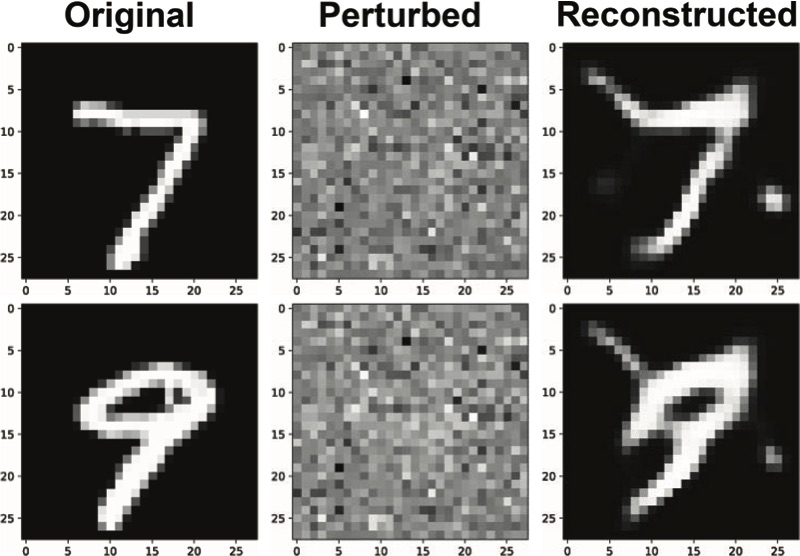}}\ \
\subfloat[(5,10\%)]{\includegraphics[width=1.58in]{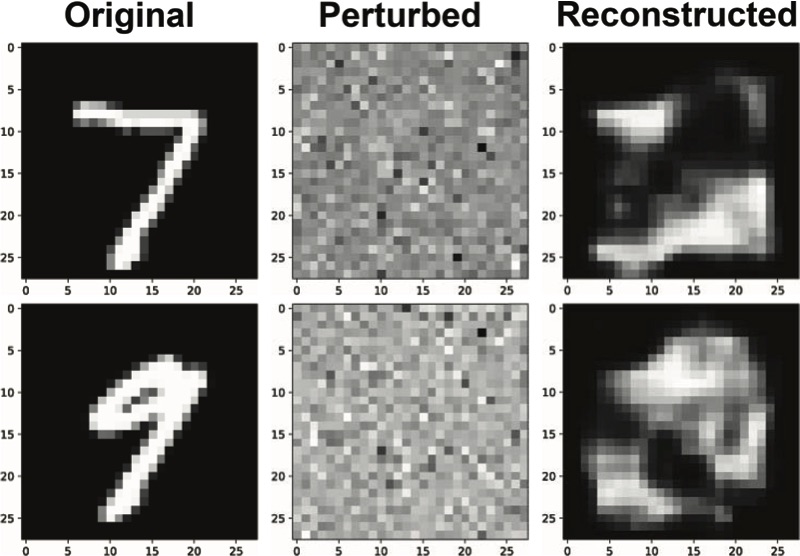}}
\caption{Visualization of noise and reconstruction.} \label{fig:reconstruct}
\end{figure}

\begin{table}[tb]
  \caption{Accuracy of Different Frameworks (\%)} \label{table:performance}
  \resizebox{3.350in}{!}{%
  \begin{tabular}{ccccccc}
  \toprule
     & Perturb & MNIST & SVHN & CIFAR-10&ALPH.&ACCEL.\\
    \midrule
    BASE & NO & 98.21 & 93.24 & 87.42 & 85.45 & 82.36\\
    \system-L1 & NO & 97.44 & 91.18 & 83.18 & 84.09 & 81.44\\
    \system-L1 & YES & 50.17 & 40.93 & 32.73 & 60.20 & 57.19\\
    \system & YES & 98.16 & 90.02 & 79.52& 83.55 & 80.05\\
    \bottomrule
  \end{tabular}
  }
\end{table}

Table~\ref{table:performance} lists the result of every framework. BASE achieves the highest accuracy with the sacrifice of privacy. Testing \system-L1 without perturbed data shows slightly lower accuracy than BASE. This result indicates that some useful information is lost during the local transformation. The accuracy of testing \system-L1 with perturbed data drops severely (\textit{e.g.}, 23.89\% on ALPH. and 24.25\% on ACCEL.) compared with testing \system-L1 without perturbed data. This result indicates that the model trained by the traditional training method which merely relies on clean data is not applicable to the prediction with perturbation, and it is hard to maintain accuracy while protecting user's privacy. \system\ which jointly minimizes the loss on clean data, noisy data, and perturbed noisy data, can considerably mitigate the negative impact brought by the private perturbation. Specially, the performance improvement demonstrated in the real mobile applications indicates that mobile device users can benefit from the cloud resources without risking their privacy in reality.

Notice that we do not make comparison with other advanced DNNs like DenseNet \cite{Huang2016}, because we intend to examine the performance improvement brought by the noisy training. The cloud-side DNN can be replaced by any other advanced DNNs to achieve higher accuracy. Specially, thanks to the transparency property of \system, the model can switch seamlessly from one to another.

\subsection{Analysis of Privacy Budget}
Privacy budget represents the privacy loss in the framework. To analyze the impact of privacy budgets on performance, we use the real mobile application DeepMood to test \system's performance when the privacy budget changes. In this work, the privacy budget is controlled by two parameters, i.e, $b$ and $\mu$. Here, we fix the value of $\mu$ as 10\%, and change the privacy budget $\epsilon$ by adjusting $b$ from 0.3 to 5.

\begin{figure}[tb]
\centering
\subfloat[ALPH.]{\includegraphics[width=1.55in]{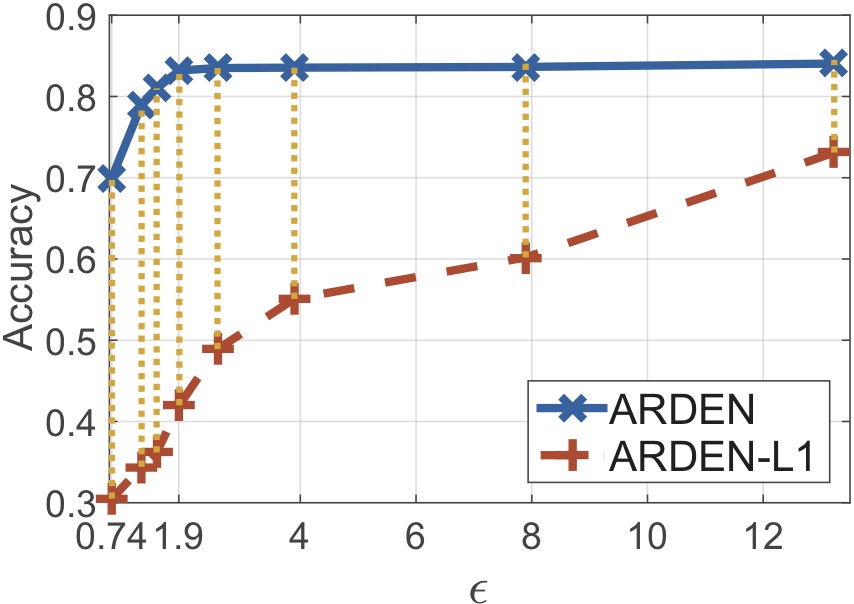}}\ \
\subfloat[ACCEL.]{\includegraphics[width=1.55in]{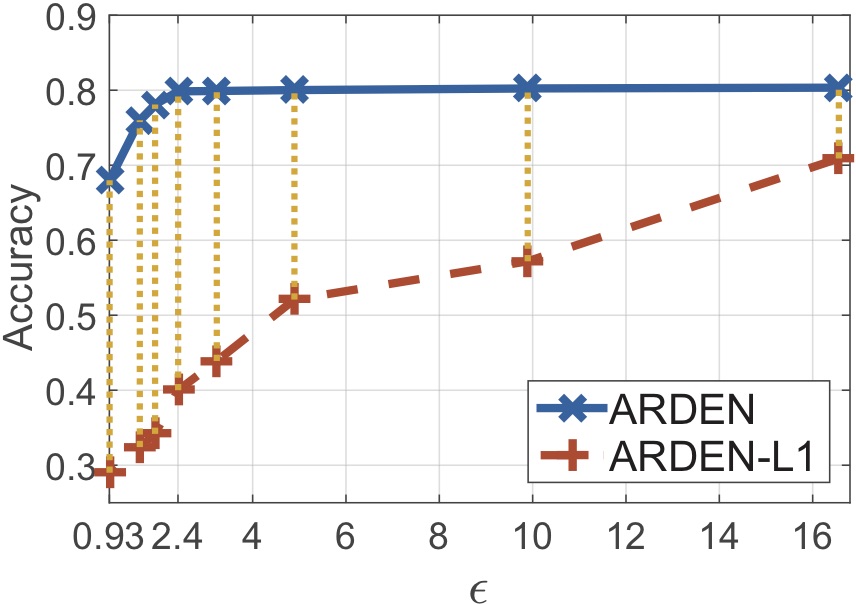}}
\caption{Accuracy vs. privacy budget $\epsilon$.} \label{fig:budget}
\end{figure}

\begin{figure}[tb]
\centering
\includegraphics[width=2.6in]{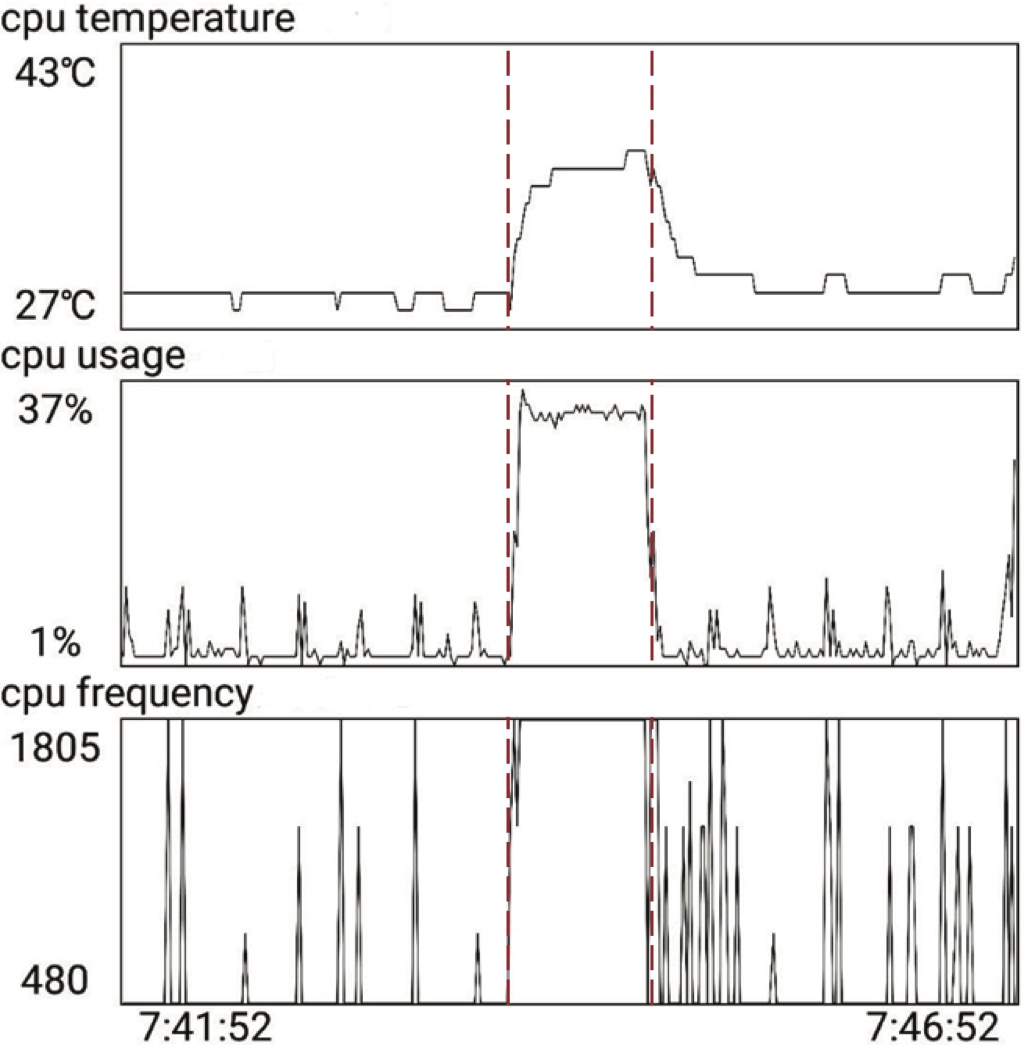}
\caption{CPU status during local transformation.} \label{fig:cpu}
\end{figure}

Fig. \ref{fig:budget} shows that \system\ can maintain the accuracy at a high value for a wide range of privacy budgets. When $\epsilon<1$, the accuracy decreases by 14\% and 12\% for ALPH. and ACCEL., respectively. The accuracy almost keeps unchanged until $\epsilon$ decreases to 1.9 for ALPH. and 2.4 for ACCEL., respectively, which is a quite tight privacy requirement \cite{Papernot2017}. These results argue that \system\ is applicable to different privacy requirements. It can effectively improve the performance even when the privacy budget is relatively tight. In addition, we plot the accuracy of \system-L1 when it is tested with private perturbation. The gap between the two lines can be viewed as the accuracy improvement brought by the noisy training. We find that the gap is large when $\epsilon$ is small and narrows with the increase of $\epsilon$ in both datasets.

\subsection{Implementation on Android}
We implement the \system\ in a demo system composed by HUAWEI HONOR 8 and DELL INSPIRON 15. The mobile device is equipped with ARM Cortex-A53@2.3GHz and ARM Cortex-A53@1.81GHz. The laptop is equipped with Core i7-7700HQ@2.80GHz and NVIDIA GTX 1050Ti. The mobile device is connected to the laptop through the IEEE 802.11 wireless network. We use TensorFlow to generate the deployable model for the Android system.

Firstly, we investigate the CPU status of the mobile device during the local transformation. 3,000 SVHN pictures are processed consecutively. Fig.~\ref{fig:cpu} shows that the CPU load boosts during the local data transformation. Especially, the CPU temperature increases by 9 degrees, which is an indicator of high energy cost. This observation argues that executing a small neural network for a relatively long time can place a heavy burden on mobile devices. It is not suitable for most mobile devices to deploy DNNs directly.

\begin {table}[tb]
 \caption{Execution Overhead Comparison}
 \label{table:overhead}
\begin{tabular}{cccc}
  \toprule
    Network & Time(ms) & Memory(MB) & Energy(J)\\
    \midrule
    Local & 3267 & 0.56 & 1.29\\
    Conv-Middle & 7386 & 5.86 & 5.07\\
    Conv-Large & 12195 & 12.89 & 8.37\\
    MobileNet & 4106 & 3.34 & 2.82\\
    GoogLeNet & 36251 & 53.12 & 24.87\\
  \bottomrule
\multicolumn{4}{c}{%
  \begin{minipage}{6.5cm}%
 Note: The time and energy cost contain the part spent on transmitting data and waiting for the response from the laptop.
  \end{minipage}%
}\\
\end{tabular}
\end{table}

We further compare the overhead of the local neural network with those of other four DNNs. MobileNet \cite{Andrew2017} is a lightweight model for mobile vision applications. GoogLeNet \cite{Szegedy2015} is a deep convolutional neural network that achieved the state of the art in ILSVRC'14. For the local neural network, the mobile device only undertakes the local transformation, and then transmits the data to the laptop, as \system\ does. The other four DNNs entirely rely on the mobile device. The networks process 100 SVHN pictures consecutively. Table~\ref{table:overhead} lists the response time and the resource consumption of different networks. We estimate the energy consumption based on the statistics from \cite{CPU2015, Balasubramanian2009}. The average reductions of \system\ compared with the other four DNNs in terms of time, memory, and energy are 60.10\%, 92.07\%, and 77.05\%, respectively.

\section{Related Work}
\textbf{Deep learning on mobile devices}. Lane et al. \cite{Lane2015} studied typical mobile sensing tasks using DNNs. The preliminary results highlighted the critical need for further research towards making use of advances in deep learning to the field of mobile sensing. However, suffered from the contradiction between the large size of DNNs and the limited capacity of mobile devices, it is challenging to deploy deep learning applications in mobile devices efficiently. Han et al. \cite{Han2015} tried to compress the DNN through a three-stage method: pruning, trained quantization and Huffman coding, which showed a considerable reduction in terms of the storage requirements of DNNs. Offloading the heavy computational tasks to the cloud data center \cite{Golkarifard2017, Liu2017} is another solution to enable deep learning applications on mobile devices. A distributed DNN architecture across the cloud, the edge, and the mobile devices was designed in \cite{Teerapittayanon2017} which allowed the combination of fast inference on mobile devices and complex inference in cloud data centers.

\textbf{Privacy issue in deep learning}. Deep learning naturally requires users' data to train neural networks and infer results, which raises the privacy issue for sensitive data. To preserve privacy when offloading operations to the cloud, Zhang et al. \cite{Zhang2016} encrypted the sensitive data by using the BGV encryption scheme. More recently, Osia et al. \cite{Osia2017} designed a hybrid deep learning architecture for private inference across mobile devices and clouds. The Siamese network was used to protect against undesired inference attacks, which in essence provided $k$-anonymity protection. Li et al. \cite{Li2017} proposed a flexible framework for private deep learning. Before uploading data to clouds, it was transformed by the local neural network whose structure, including the number of layers, the depth of output channels, and the subset of selected channels, was variable. Differential privacy mechanism which provides provable privacy guarantee has been used in deep learning. Shokri et al. \cite{Shokri2015} presented a privacy-preserving distributed SGD to enable multiple data owners to collaborate on training a DNN. The sparse vector technique was introduced in this work to provide differential privacy. Abadi et al. \cite{Abadi2016} designed a new differential privacy mechanism in SGD to reduce the privacy budget. However, these works applied differential privacy to the training phase rather than the inference phase. To the best of our knowledge, this is the first work using differential privacy in the inference phase to provide a provable privacy guarantee.

\section{Conclusions}
In order to enable highly efficient deep learning service on mobile devices, we propose \system\ which partitions the DNN across mobile devices and clouds in this paper. All the heavy works are offloaded to the cloud, while the mobile device merely undertakes the simple data transformation and perturbation. To limit the privacy risk when uploading data to clouds, we introduce the differential privacy mechanism, and design a new differentially private perturbation that shows more flexibility and fits well with the stacking structure of neural networks. A rigorous analysis is given to determine the privacy budget of the perturbation. Apart from privacy, the inference performance is taken into consideration in this paper. A novel training method that injects deliberate noisy samples into the training data is proposed. A series of experiments based on three image datasets and a real mobile application demonstrate that \system\ can not only preserve users' privacy but also improve the inference performance. Finally, we implement \system\ on a demo mobile cloud system to verify its practicality and test its overhead. \system\ effectively reduces the resource consumption by over 60\%.

\appendix
\section{Proof of Theorem~\ref{thm:privacy}}
Before proving Theorem~\ref{thm:privacy}, we first prove the following two theorems.
\begin{theorem}
Given a input $x$ and a deterministic function $f$, $|f(x)|\leq B$, for $\forall a \in \mathbb{R}^+$, the random mechanism $\mathcal{A}(x)=f(x)+aLap(B/\sigma)$ is $(\frac{{2\sigma }}{a})$-differentially private.
\end{theorem}

\begin{proof}
For any adjacent inputs $x$ and $x'$,
\begin{equation*}
\begin{split}
& \frac{{\Pr [f(x) + aLap(B/\sigma ) = S]}}{{\Pr [f(x') + aLap(B/\sigma ) = S]}} = \frac{{{e^{ - \frac{{|S - f(x)|\sigma }}{{aB}}}}}}{{{e^{ - \frac{{|S - f(x')|\sigma }}{{aB}}}}}}\\
& = {e^{\frac{\sigma }{{aB}}(|S - f(x')| - |S - f(x)|)}} \\
& \le {e^{\frac{\sigma }{{aB}}|f(x) - f(x')|}} \le {e^{\frac{{2\sigma }}{a}}}
\end{split}
\end{equation*}

Based on the Definition 2.1, we have $\varepsilon=\frac{{2\sigma}}{a}$.
\end{proof}

\begin{theorem}
Given a input $x$, suppose $\mathcal{A}(x)$ is $\varepsilon$-differentially private, if we perform item-wise nullification of $x$ with $I_n$ whose nullification rate is $\mu$, $x'=x\odot I_n$, then $\mathcal{A}(x')$ is $\varepsilon '$-differentially private,
\begin{equation*}
\varepsilon ' = \ln[(1-\mu)e^\varepsilon +\mu].
\end{equation*}
\end{theorem}

\begin{proof}
Suppose there are two adjacent inputs $x_1$ and $x_2$ that differ by one single item $i$, say $x_1=x_2\cup {i}$. For arbitrary binary matrix $I_n$, after the nullification $x'_1=x_1\odot I_n$, $x'_2=x_2\odot I_n$, there are two possible cases, \textit{i.e.}, $i \notin x'_1$ and $i \in x'_1$.

Case 1: $i \notin x'_1$. Since $x_1$ and $x_2$ differ only by the item $i$, $x_1\odot I_n$ = $x_2\odot I_n$. Then, we have,
\begin{equation*}
\Pr[\mathcal{A}(x_1\odot I_n)=S]=\Pr[\mathcal{A}(x_2\odot I_n)=S].
\end{equation*}

Case 2: $i \in x'_1$. Since $x_1$ and $x_2$ differ only by the item $i$, $x'_1$ and $x'_2$ remain adjacent inputs that differ by the item $i$. Because $\mathcal{A}(x)$ is $\varepsilon$-differentially private, we have,
\begin{equation*}
\Pr[\mathcal{A}(x_1\odot I_n)=S]\leq e^{\varepsilon}\Pr[\mathcal{A}(x_2\odot I_n)=S].
\end{equation*}

Combine the two cases together, and use the fact that $\Pr[i \notin x'_1]=\mu$:
\begin{equation*}
\begin{split}
& \Pr [\mathcal{A}(x_1\odot I_n)=S]\\
& = \mu \Pr[\mathcal{A}(x_1\odot I_n)=S] + (1-\mu) \Pr[\mathcal{A}(x_1\odot I_n)=S] \\
& \leq \mu \Pr[\mathcal{A}(x_2\odot I_n)=S] + (1-\mu) e^{\varepsilon}\Pr[\mathcal{A}(x_2\odot I_n)=S] \\
& = ((1-\mu) e^{\varepsilon}+\mu)\Pr[\mathcal{A}(x_2\odot I_n)=S] \\
& = e^{\ln[(1-\mu)e^\varepsilon +\mu]} \Pr[\mathcal{A}(x_2\odot I_n)=S]
\end{split}
\end{equation*}

Based on the Definition 2.1, we have $\varepsilon '=\ln[(1-\mu)e^\varepsilon +\mu]$.
\end{proof}

Then, we give the proof of Theorem~\ref{thm:privacy}.

\begin{proof}
To begin with, we analyze the local transformation without the nullification operation. Denote the output-bound neural network $\mathcal{M}_l$ as $\mathcal{M}'_l$, the local transformation $\mathcal{A}$ defined by Algorithm 1 can be written as:
\begin{equation}\label{eqn:A}
\begin{split}
\mathcal{A}(\boldsymbol{x}_s) & =\overline{\mathcal{M}}_l(\mathcal{M}'_l(\boldsymbol{x}_s)+Lap(B/\sigma \mathbf{I}))\\
& = \overline{\mathcal{M}}_l(\boldsymbol{x}'_l+Lap(B/\sigma \mathbf{I})).
\end{split}
\end{equation}

Since $Lap(B/\sigma \mathbf{I})$ is much smaller than $\boldsymbol{x}'_l$ (otherwise, the utility of the representation is destroyed), we can approximate (\ref{eqn:A}) by using the first-order Taylor series:
\begin{equation}\label{eqn:A-T}
\begin{split}
& \overline{\mathcal{M}}_l(\boldsymbol{x}'_l+Lap(B/\sigma \mathbf{I})) \\
& \approx \overline{\mathcal{M}}_l(\boldsymbol{x}'_l) + ({\nabla _{{{\boldsymbol{x}}'_l}}}{\overline {\mathcal M} _l})^{\mathrm{T}}Lap(B/\sigma \mathbf{I})\\
& = \mathcal{M}(\boldsymbol{x}_s) + ({\nabla _{{{\boldsymbol{x}}'_l}}}{\overline {\mathcal M} _l})^{\mathrm{T}}Lap(B/\sigma \mathbf{I}).
\end{split}
\end{equation}

According to the Theorem A.1, (\ref{eqn:A-T}) is $(2\sigma / \Lambda)$-differentially private, where $\Lambda = \|{\nabla _{{{\boldsymbol{x}}'_l}}}{\overline {\mathcal M} _l}\|_{\infty}$.

Then, considering the nullification operation, the overall local transformation is $\mathcal{A}(\boldsymbol{x}_s\odot I_n)$. Based on the Theorem A.2, we have,
\begin{equation*}
\varepsilon  = \ln [(1-\mu) {e^{2\sigma /\Lambda}} + \mu].
\end{equation*}

\end{proof}

\bibliographystyle{ACM-Reference-Format}
\bibliography{sample-bibliography}


\begin{thebibliography}{47}


\ifx \showCODEN    \undefined \def \showCODEN     #1{\unskip}     \fi
\ifx \showDOI      \undefined \def \showDOI       #1{#1}\fi
\ifx \showISBNx    \undefined \def \showISBNx     #1{\unskip}     \fi
\ifx \showISBNxiii \undefined \def \showISBNxiii  #1{\unskip}     \fi
\ifx \showISSN     \undefined \def \showISSN      #1{\unskip}     \fi
\ifx \showLCCN     \undefined \def \showLCCN      #1{\unskip}     \fi
\ifx \shownote     \undefined \def \shownote      #1{#1}          \fi
\ifx \showarticletitle \undefined \def \showarticletitle #1{#1}   \fi
\ifx \showURL      \undefined \def \showURL       {\relax}        \fi
\providecommand\bibfield[2]{#2}
\providecommand\bibinfo[2]{#2}
\providecommand\natexlab[1]{#1}
\providecommand\showeprint[2][]{arXiv:#2}

\bibitem[\protect\citeauthoryear{Abadi, Chu, Goodfellow, McMahan, Mironov,
  Talwar, and Zhang}{Abadi et~al\mbox{.}}{2016}]%
        {Abadi2016}
\bibfield{author}{\bibinfo{person}{Mart\'{\i}n Abadi}, \bibinfo{person}{Andy
  Chu}, \bibinfo{person}{Ian Goodfellow}, \bibinfo{person}{H.~Brendan McMahan},
  \bibinfo{person}{Ilya Mironov}, \bibinfo{person}{Kunal Talwar}, {and}
  \bibinfo{person}{Li Zhang}.} \bibinfo{year}{2016}\natexlab{}.
\newblock \showarticletitle{Deep Learning with Differential Privacy}. In
  \bibinfo{booktitle}{\emph{Proceedings of the 23rd ACM SIGSAC Conference on
  Computer and Communications Security (CCS)}}. \bibinfo{pages}{308--318}.
\newblock


\bibitem[\protect\citeauthoryear{Abadi and et~al}{Abadi and et~al}{2015}]%
        {Tensorflow}
\bibfield{author}{\bibinfo{person}{Mart\'{\i}n Abadi} {and} \bibinfo{person}{et
  al}.} \bibinfo{year}{2015}\natexlab{}.
\newblock \bibinfo{title}{{TensorFlow}: Large-Scale Machine Learning on
  Heterogeneous Systems}.
\newblock
\newblock
\urldef\tempurl%
\url{http://tensorflow.org/}
\showURL{%
\tempurl}
\newblock
\shownote{Software available from tensorflow.org.}


\bibitem[\protect\citeauthoryear{Andrei and Ryan}{Andrei and Ryan}{2015}]%
        {CPU2015}
\bibfield{author}{\bibinfo{person}{Frumusanu Andrei} {and}
  \bibinfo{person}{Smith Ryan}.} \bibinfo{year}{2015}\natexlab{}.
\newblock \showarticletitle{Cortex A53 - performance and power - ARM
  A53/A57/T760 investigated}.
\newblock  (\bibinfo{year}{2015}).
\newblock


\bibitem[\protect\citeauthoryear{Balasubramanian, Balasubramanian, and
  Venkataramani}{Balasubramanian et~al\mbox{.}}{2009}]%
        {Balasubramanian2009}
\bibfield{author}{\bibinfo{person}{Niranjan Balasubramanian},
  \bibinfo{person}{Aruna Balasubramanian}, {and} \bibinfo{person}{Arun
  Venkataramani}.} \bibinfo{year}{2009}\natexlab{}.
\newblock \showarticletitle{Energy Consumption in Mobile Phones: A Measurement
  Study and Implications for Network Applications}. In
  \bibinfo{booktitle}{\emph{Proceedings of the 9th ACM SIGCOMM Conference on
  Internet Measurement (IMC)}}. \bibinfo{pages}{280--293}.
\newblock


\bibitem[\protect\citeauthoryear{Beimel, Brenner, Kasiviswanathan, and
  Nissim}{Beimel et~al\mbox{.}}{2014}]%
        {Beimel2014}
\bibfield{author}{\bibinfo{person}{Amos Beimel}, \bibinfo{person}{Hai Brenner},
  \bibinfo{person}{Shiva~Prasad Kasiviswanathan}, {and} \bibinfo{person}{Kobbi
  Nissim}.} \bibinfo{year}{2014}\natexlab{}.
\newblock \showarticletitle{Bounds on the sample complexity for private
  learning and private data release}.
\newblock \bibinfo{journal}{\emph{Machine Learning}} \bibinfo{volume}{94},
  \bibinfo{number}{3} (\bibinfo{year}{2014}), \bibinfo{pages}{401--437}.
\newblock


\bibitem[\protect\citeauthoryear{Cao, Zheng, Zhang, Yu, Piscitello, Zulueta,
  Ajilore, Ryan, and Leow}{Cao et~al\mbox{.}}{2017}]%
        {Cao2017}
\bibfield{author}{\bibinfo{person}{Bokai Cao}, \bibinfo{person}{Lei Zheng},
  \bibinfo{person}{Chenwei Zhang}, \bibinfo{person}{Philip~S. Yu},
  \bibinfo{person}{Andrea Piscitello}, \bibinfo{person}{John Zulueta},
  \bibinfo{person}{Olu Ajilore}, \bibinfo{person}{Kelly Ryan}, {and}
  \bibinfo{person}{Alex~D. Leow}.} \bibinfo{year}{2017}\natexlab{}.
\newblock \showarticletitle{DeepMood: Modeling Mobile Phone Typing Dynamics for
  Mood Detection}. In \bibinfo{booktitle}{\emph{Proceedings of ACM SIGKDD
  Conference on Knowledge Discovery and Data Mining (KDD)}}.
  \bibinfo{pages}{747--755}.
\newblock


\bibitem[\protect\citeauthoryear{Ding, Liao, Wang, Li, Liu, Zhuo, Wang, Qian,
  Bai, Yuan, Ma, Zhang, Tang, Qiu, Lin, and Yuan}{Ding et~al\mbox{.}}{2017}]%
        {Ding2017}
\bibfield{author}{\bibinfo{person}{Caiwen Ding}, \bibinfo{person}{Siyu Liao},
  \bibinfo{person}{Yanzhi Wang}, \bibinfo{person}{Zhe Li},
  \bibinfo{person}{Ning Liu}, \bibinfo{person}{Youwei Zhuo},
  \bibinfo{person}{Chao Wang}, \bibinfo{person}{Xuehai Qian},
  \bibinfo{person}{Yu Bai}, \bibinfo{person}{Geng Yuan},
  \bibinfo{person}{Xiaolong Ma}, \bibinfo{person}{Yipeng Zhang},
  \bibinfo{person}{Jian Tang}, \bibinfo{person}{Qinru Qiu},
  \bibinfo{person}{Xue Lin}, {and} \bibinfo{person}{Bo Yuan}.}
  \bibinfo{year}{2017}\natexlab{}.
\newblock \showarticletitle{CirCNN: Accelerating and Compressing Deep Neural
  Networks Using Block-CirculantWeight Matrices}.
\newblock \bibinfo{journal}{\emph{arXiv:1708.08917}} (\bibinfo{year}{2017}).
\newblock
\urldef\tempurl%
\url{https://doi.org/abs/1708.08917}
\showDOI{\tempurl}


\bibitem[\protect\citeauthoryear{Dong, Loy, He, and Tang}{Dong
  et~al\mbox{.}}{2016}]%
        {Dong2016}
\bibfield{author}{\bibinfo{person}{Chao Dong}, \bibinfo{person}{Chen~Change
  Loy}, \bibinfo{person}{Kaiming He}, {and} \bibinfo{person}{Xiaoou Tang}.}
  \bibinfo{year}{2016}\natexlab{}.
\newblock \showarticletitle{Image Super-Resolution Using Deep Convolutional
  Networks}.
\newblock \bibinfo{journal}{\emph{IEEE Transactions on Pattern Analysis and
  Machine Intelligence}} \bibinfo{volume}{38}, \bibinfo{number}{2}
  (\bibinfo{year}{2016}), \bibinfo{pages}{295--307}.
\newblock


\bibitem[\protect\citeauthoryear{Dwork}{Dwork}{2011a}]%
        {Dwork2011diff}
\bibfield{author}{\bibinfo{person}{Cynthia Dwork}.}
  \bibinfo{year}{2011}\natexlab{a}.
\newblock \bibinfo{booktitle}{\emph{Differential Privacy}}.
\newblock \bibinfo{publisher}{Springer US}, \bibinfo{address}{Boston, MA},
  \bibinfo{pages}{338--340}.
\newblock


\bibitem[\protect\citeauthoryear{Dwork}{Dwork}{2011b}]%
        {Dwork2011}
\bibfield{author}{\bibinfo{person}{Cynthia Dwork}.}
  \bibinfo{year}{2011}\natexlab{b}.
\newblock \showarticletitle{A Firm Foundation for Private Data Analysis}.
\newblock \bibinfo{journal}{\emph{Commun. ACM}} \bibinfo{volume}{54},
  \bibinfo{number}{1} (\bibinfo{year}{2011}), \bibinfo{pages}{86--95}.
\newblock


\bibitem[\protect\citeauthoryear{Dwork and Roth}{Dwork and Roth}{2014}]%
        {Dwork2013}
\bibfield{author}{\bibinfo{person}{Cynthia Dwork} {and} \bibinfo{person}{Aaron
  Roth}.} \bibinfo{year}{2014}\natexlab{}.
\newblock \showarticletitle{The Algorithmic Foundations of Differential
  Privacy}.
\newblock \bibinfo{journal}{\emph{Foundations and Trends in Theoretical
  Computer Science}} \bibinfo{volume}{9}, \bibinfo{number}{3\&\#8211;4}
  (\bibinfo{year}{2014}), \bibinfo{pages}{211--407}.
\newblock


\bibitem[\protect\citeauthoryear{Golkarifard, Yang, Movaghar, and
  Hui}{Golkarifard et~al\mbox{.}}{2017}]%
        {Golkarifard2017}
\bibfield{author}{\bibinfo{person}{Morteza Golkarifard}, \bibinfo{person}{Ji
  Yang}, \bibinfo{person}{Ali Movaghar}, {and} \bibinfo{person}{Pan Hui}.}
  \bibinfo{year}{2017}\natexlab{}.
\newblock \showarticletitle{A Hitchhiker's Guide to Computation Offloading:
  Opinions from Practitioners}.
\newblock \bibinfo{journal}{\emph{IEEE Communications Magazine}}
  \bibinfo{volume}{55}, \bibinfo{number}{7} (\bibinfo{year}{2017}),
  \bibinfo{pages}{193--199}.
\newblock


\bibitem[\protect\citeauthoryear{Han, Mao, and Dally}{Han
  et~al\mbox{.}}{2016}]%
        {Han2015}
\bibfield{author}{\bibinfo{person}{Song Han}, \bibinfo{person}{Huizi Mao},
  {and} \bibinfo{person}{William~J. Dally}.} \bibinfo{year}{2016}\natexlab{}.
\newblock \showarticletitle{Deep Compression: Compressing Deep Neural Networks
  with Pruning, Trained Quantization and Huffman Coding}. In
  \bibinfo{booktitle}{\emph{4th International Conference on Learning
  Representations (ICLR)}}.
\newblock


\bibitem[\protect\citeauthoryear{Han, Pool, Tran, and Dally}{Han
  et~al\mbox{.}}{2015}]%
        {Han2015learning}
\bibfield{author}{\bibinfo{person}{Song Han}, \bibinfo{person}{Jeff Pool},
  \bibinfo{person}{John Tran}, {and} \bibinfo{person}{William~J. Dally}.}
  \bibinfo{year}{2015}\natexlab{}.
\newblock \showarticletitle{Learning Both Weights and Connections for Efficient
  Neural Networks}. In \bibinfo{booktitle}{\emph{Proceedings of the 28th
  International Conference on Neural Information Processing Systems (NIPS)}}.
  \bibinfo{pages}{1135--1143}.
\newblock


\bibitem[\protect\citeauthoryear{He, Gkioxari, Doll$\acute{a}$r, and
  Girshick}{He et~al\mbox{.}}{2017}]%
        {He2017}
\bibfield{author}{\bibinfo{person}{Kaiming He}, \bibinfo{person}{Georgia
  Gkioxari}, \bibinfo{person}{Piotr Doll$\acute{a}$r}, {and}
  \bibinfo{person}{Ross Girshick}.} \bibinfo{year}{2017}\natexlab{}.
\newblock \showarticletitle{Mask R-CNN}. In
  \bibinfo{booktitle}{\emph{Proceedings of the IEEE International Conference on
  Computer Vision (ICCV)}}.
\newblock


\bibitem[\protect\citeauthoryear{He, Zhang, Ren, and Sun}{He
  et~al\mbox{.}}{2016}]%
        {He2015}
\bibfield{author}{\bibinfo{person}{Kaiming He}, \bibinfo{person}{Xiangyu
  Zhang}, \bibinfo{person}{Shaoqing Ren}, {and} \bibinfo{person}{Jian Sun}.}
  \bibinfo{year}{2016}\natexlab{}.
\newblock \showarticletitle{Deep Residual Learning for Image Recognition}. In
  \bibinfo{booktitle}{\emph{IEEE Conference on Computer Vision and Pattern
  Recognition (CVPR)}}. \bibinfo{pages}{770--778}.
\newblock


\bibitem[\protect\citeauthoryear{Hinton and Salakhutdinov}{Hinton and
  Salakhutdinov}{2006}]%
        {Hinton2006}
\bibfield{author}{\bibinfo{person}{G.~E. Hinton} {and} \bibinfo{person}{R.~R.
  Salakhutdinov}.} \bibinfo{year}{2006}\natexlab{}.
\newblock \showarticletitle{Reducing the Dimensionality of Data with Neural
  Networks}.
\newblock \bibinfo{journal}{\emph{Science}} \bibinfo{volume}{313},
  \bibinfo{number}{5786} (\bibinfo{year}{2006}), \bibinfo{pages}{504--507}.
\newblock


\bibitem[\protect\citeauthoryear{Howard, Zhu, Chen, Kalenichenko, Wang, Weyand,
  Andreetto, and Adam}{Howard et~al\mbox{.}}{2017}]%
        {Andrew2017}
\bibfield{author}{\bibinfo{person}{Andrew~G. Howard}, \bibinfo{person}{Menglong
  Zhu}, \bibinfo{person}{Bo Chen}, \bibinfo{person}{Dmitry Kalenichenko},
  \bibinfo{person}{Weijun Wang}, \bibinfo{person}{Tobias Weyand},
  \bibinfo{person}{Marco Andreetto}, {and} \bibinfo{person}{Hartwig Adam}.}
  \bibinfo{year}{2017}\natexlab{}.
\newblock \showarticletitle{MobileNets: Efficient Convolutional Neural Networks
  for Mobile Vision Applications}.
\newblock \bibinfo{journal}{\emph{arXiv:1704.04861}} (\bibinfo{year}{2017}).
\newblock
\urldef\tempurl%
\url{https://doi.org/abs/1704.04861}
\showDOI{\tempurl}


\bibitem[\protect\citeauthoryear{Huang, Liu, Weinberger, and van~der
  Maaten}{Huang et~al\mbox{.}}{2017}]%
        {Huang2016}
\bibfield{author}{\bibinfo{person}{Gao Huang}, \bibinfo{person}{Zhuang Liu},
  \bibinfo{person}{Kilian~Q. Weinberger}, {and} \bibinfo{person}{Laurens
  van~der Maaten}.} \bibinfo{year}{2017}\natexlab{}.
\newblock \showarticletitle{Densely Connected Convolutional Networks}. In
  \bibinfo{booktitle}{\emph{IEEE Conference on Computer Vision and Pattern
  Recognition (CVPR)}}.
\newblock


\bibitem[\protect\citeauthoryear{Krizhevsky and Hinton}{Krizhevsky and
  Hinton}{2009}]%
        {Krizhevsky2009}
\bibfield{author}{\bibinfo{person}{A Krizhevsky} {and} \bibinfo{person}{G
  Hinton}.} \bibinfo{year}{2009}\natexlab{}.
\newblock \showarticletitle{Learning Multiple Layers of Features from Tiny
  Images}.
\newblock  (\bibinfo{year}{2009}).
\newblock


\bibitem[\protect\citeauthoryear{Krizhevsky, Sutskever, and Hinton}{Krizhevsky
  et~al\mbox{.}}{2012}]%
        {Krizhevsky2012}
\bibfield{author}{\bibinfo{person}{Alex Krizhevsky}, \bibinfo{person}{Ilya
  Sutskever}, {and} \bibinfo{person}{Geoffrey~E Hinton}.}
  \bibinfo{year}{2012}\natexlab{}.
\newblock \showarticletitle{ImageNet Classification with Deep Convolutional
  Neural Networks}.
\newblock In \bibinfo{booktitle}{\emph{Proceedings of the 25th International
  Conference on Neural Information Processing Systems (NIPS)}}.
  \bibinfo{pages}{1097--1105}.
\newblock


\bibitem[\protect\citeauthoryear{Laine}{Laine}{2017}]%
        {Laine2017}
\bibfield{author}{\bibinfo{person}{Samuli Laine}.}
  \bibinfo{year}{2017}\natexlab{}.
\newblock \showarticletitle{Temporal Ensembling for Semi-Supervised Learning}.
  In \bibinfo{booktitle}{\emph{5th International Conference on Learning
  Representations (ICLR)}}.
\newblock


\bibitem[\protect\citeauthoryear{Lane, Bhattacharya, Georgiev, Forlivesi, Jiao,
  Qendro, and Kawsar}{Lane et~al\mbox{.}}{2016}]%
        {Lane2016}
\bibfield{author}{\bibinfo{person}{Nicholas~D. Lane}, \bibinfo{person}{Sourav
  Bhattacharya}, \bibinfo{person}{Petko Georgiev}, \bibinfo{person}{Claudio
  Forlivesi}, \bibinfo{person}{Lei Jiao}, \bibinfo{person}{Lorena Qendro},
  {and} \bibinfo{person}{Fahim Kawsar}.} \bibinfo{year}{2016}\natexlab{}.
\newblock \showarticletitle{DeepX: A Software Accelerator for Low-Power Deep
  Learning Inference on Mobile Devices}. In \bibinfo{booktitle}{\emph{15th
  ACM/IEEE International Conference on Information Processing in Sensor
  Networks (IPSN)}}. \bibinfo{pages}{1--12}.
\newblock


\bibitem[\protect\citeauthoryear{Lane and Georgiev}{Lane and Georgiev}{2015}]%
        {Lane2015}
\bibfield{author}{\bibinfo{person}{Nicholas~D. Lane} {and}
  \bibinfo{person}{Petko Georgiev}.} \bibinfo{year}{2015}\natexlab{}.
\newblock \showarticletitle{Can Deep Learning Revolutionize Mobile Sensing?}.
  In \bibinfo{booktitle}{\emph{Proceedings of the 16th International Workshop
  on Mobile Computing Systems and Applications (HotMobile)}}.
  \bibinfo{pages}{117--122}.
\newblock


\bibitem[\protect\citeauthoryear{LeCun, Bottou, Bengio, and Hinton}{LeCun
  et~al\mbox{.}}{1998}]%
        {LeCun1998}
\bibfield{author}{\bibinfo{person}{Yann LeCun}, \bibinfo{person}{L$\acute{e}$on
  Bottou}, \bibinfo{person}{Yoshua Bengio}, {and} \bibinfo{person}{Geoffrey
  Hinton}.} \bibinfo{year}{1998}\natexlab{}.
\newblock \showarticletitle{Gradient-based learning applied to document
  recognition}.
\newblock \bibinfo{journal}{\emph{Proc. IEEE}} \bibinfo{volume}{86},
  \bibinfo{number}{11} (\bibinfo{year}{1998}), \bibinfo{pages}{2278--2324}.
\newblock


\bibitem[\protect\citeauthoryear{Ledig, Theis, Huszar, Caballero, Cunningham,
  Acosta, Aitken, Tejani, Totz, Wang, and Shi}{Ledig et~al\mbox{.}}{2017}]%
        {Ledig2017}
\bibfield{author}{\bibinfo{person}{Christian Ledig}, \bibinfo{person}{Lucas
  Theis}, \bibinfo{person}{Ferenc Huszar}, \bibinfo{person}{Jose Caballero},
  \bibinfo{person}{Andrew Cunningham}, \bibinfo{person}{Alejandro Acosta},
  \bibinfo{person}{Andrew Aitken}, \bibinfo{person}{Alykhan Tejani},
  \bibinfo{person}{Johannes Totz}, \bibinfo{person}{Zehan Wang}, {and}
  \bibinfo{person}{Wenzhe Shi}.} \bibinfo{year}{2017}\natexlab{}.
\newblock \showarticletitle{Photo-Realistic Single Image Super-Resolution Using
  a Generative Adversarial Network}. In \bibinfo{booktitle}{\emph{IEEE
  Conference on Computer Vision and Pattern Recognition (CVPR)}}.
  \bibinfo{pages}{4681--4690}.
\newblock


\bibitem[\protect\citeauthoryear{Lee}{Lee}{2017}]%
        {Delloitte2017}
\bibfield{author}{\bibinfo{person}{Paul Lee}.} \bibinfo{year}{2017}\natexlab{}.
\newblock \showarticletitle{Technology, Media and Telecommunications
  Predictions}.
\newblock \bibinfo{journal}{\emph{Delloitte Touche Tohmatsu Limited}}
  (\bibinfo{year}{2017}).
\newblock


\bibitem[\protect\citeauthoryear{Li, Xiong, Tu, Zhu, Zhang, and Zhou}{Li
  et~al\mbox{.}}{2017b}]%
        {Li2017modeling}
\bibfield{author}{\bibinfo{person}{Junhui Li}, \bibinfo{person}{Deyi Xiong},
  \bibinfo{person}{Zhaopeng Tu}, \bibinfo{person}{Muhua Zhu},
  \bibinfo{person}{Min Zhang}, {and} \bibinfo{person}{Guodong Zhou}.}
  \bibinfo{year}{2017}\natexlab{b}.
\newblock \showarticletitle{Modeling Source Syntax for Neural Machine
  Translation}. In \bibinfo{booktitle}{\emph{55th annual meeting of the
  Association for Computational Linguistics (ACL)}}.
  \bibinfo{pages}{4594--4602}.
\newblock


\bibitem[\protect\citeauthoryear{Li, Lai, Suda, Chandra, and Pan}{Li
  et~al\mbox{.}}{2017a}]%
        {Li2017}
\bibfield{author}{\bibinfo{person}{Meng Li}, \bibinfo{person}{Liangzhen Lai},
  \bibinfo{person}{Naveen Suda}, \bibinfo{person}{Vikas Chandra}, {and}
  \bibinfo{person}{David~Z. Pan}.} \bibinfo{year}{2017}\natexlab{a}.
\newblock \showarticletitle{PrivyNet: A Flexible Framework for
  Privacy-Preserving Deep Neural Network Training with A Fine-Grained Privacy
  Control}.
\newblock \bibinfo{journal}{\emph{arXiv:1709.06161}} (\bibinfo{year}{2017}).
\newblock
\urldef\tempurl%
\url{https://doi.org/abs/1709.06161}
\showDOI{\tempurl}


\bibitem[\protect\citeauthoryear{Liu, Cao, Yang, Xu, Qiu, and Li}{Liu
  et~al\mbox{.}}{2017}]%
        {Liu2017}
\bibfield{author}{\bibinfo{person}{Weiqing Liu}, \bibinfo{person}{Jiannong
  Cao}, \bibinfo{person}{Lei Yang}, \bibinfo{person}{Lin Xu},
  \bibinfo{person}{Xuanjia Qiu}, {and} \bibinfo{person}{Jing Li}.}
  \bibinfo{year}{2017}\natexlab{}.
\newblock \showarticletitle{AppBooster: Boosting the Performance of Interactive
  Mobile Applications with Computation Offloading and Parameter Tuning}.
\newblock \bibinfo{journal}{\emph{IEEE Transactions on Parallel and Distributed
  Systems}} \bibinfo{volume}{28}, \bibinfo{number}{6} (\bibinfo{year}{2017}),
  \bibinfo{pages}{1593--1606}.
\newblock


\bibitem[\protect\citeauthoryear{Masci, Meier, Cire{\c{s}}an, and
  Schmidhuber}{Masci et~al\mbox{.}}{2011}]%
        {Masci2011}
\bibfield{author}{\bibinfo{person}{Jonathan Masci}, \bibinfo{person}{Ueli
  Meier}, \bibinfo{person}{Dan Cire{\c{s}}an}, {and}
  \bibinfo{person}{J{\"u}rgen Schmidhuber}.} \bibinfo{year}{2011}\natexlab{}.
\newblock \bibinfo{booktitle}{\emph{Stacked Convolutional Auto-Encoders for
  Hierarchical Feature Extraction}}.
\newblock \bibinfo{publisher}{Springer Berlin Heidelberg},
  \bibinfo{address}{Berlin, Heidelberg}, \bibinfo{pages}{52--59}.
\newblock


\bibitem[\protect\citeauthoryear{Netzer, Wang, Coates, Bissacco, and Ng}{Netzer
  et~al\mbox{.}}{2011}]%
        {Netzer2011}
\bibfield{author}{\bibinfo{person}{Yuval Netzer}, \bibinfo{person}{Tao Wang},
  \bibinfo{person}{Adam Coates}, \bibinfo{person}{Alessandro Bissacco}, {and}
  \bibinfo{person}{Bo~Wu Andrew~Y. Ng}.} \bibinfo{year}{2011}\natexlab{}.
\newblock \showarticletitle{Reading Digits in Natural Images with Unsupervised
  Feature Learning}.
\newblock In \bibinfo{booktitle}{\emph{Proceedings of the 24th International
  Conference on Neural Information Processing Systems (NIPS)}}.
  \bibinfo{pages}{1--9}.
\newblock


\bibitem[\protect\citeauthoryear{Nicolas~Papernot}{Nicolas~Papernot}{2017}]%
        {Papernot2017}
\bibfield{author}{\bibinfo{person}{Ulfar Erlingsson Ian Goodfellow Kunal~Talwar
  Nicolas~Papernot, Martin~Abadi}.} \bibinfo{year}{2017}\natexlab{}.
\newblock \showarticletitle{Semi-supervised Knowledge Transfer for Deep
  Learning from Private Training Data}. In \bibinfo{booktitle}{\emph{5th
  International Conference on Learning Representations (ICLR)}}.
\newblock


\bibitem[\protect\citeauthoryear{Oquab, Bottou, Laptev, and Sivic}{Oquab
  et~al\mbox{.}}{2014}]%
        {Oquab2014}
\bibfield{author}{\bibinfo{person}{Maxime Oquab}, \bibinfo{person}{Leon
  Bottou}, \bibinfo{person}{Ivan Laptev}, {and} \bibinfo{person}{Josef Sivic}.}
  \bibinfo{year}{2014}\natexlab{}.
\newblock \showarticletitle{Learning and Transferring Mid-Level Image
  Representations using Convolutional Neural Networks}. In
  \bibinfo{booktitle}{\emph{IEEE Conference on Computer Vision and Pattern
  Recognition (CVPR)}}. \bibinfo{pages}{1717--1724}.
\newblock


\bibitem[\protect\citeauthoryear{Osia, Shamsabadi, Taheri, Rabiee, Lane, and
  Haddadi}{Osia et~al\mbox{.}}{2017}]%
        {Osia2017}
\bibfield{author}{\bibinfo{person}{Seyed~Ali Osia}, \bibinfo{person}{Ali~Shahin
  Shamsabadi}, \bibinfo{person}{Ali Taheri}, \bibinfo{person}{Hamid~R. Rabiee},
  \bibinfo{person}{Nicholas~D. Lane}, {and} \bibinfo{person}{Hamed Haddadi}.}
  \bibinfo{year}{2017}\natexlab{}.
\newblock \showarticletitle{A Hybrid Deep Learning Architecture for
  Privacy-Preserving Mobile Analytics}.
\newblock \bibinfo{journal}{\emph{arXiv:1703.02952}} (\bibinfo{year}{2017}).
\newblock
\urldef\tempurl%
\url{https://doi.org/abs/1703.02952}
\showDOI{\tempurl}


\bibitem[\protect\citeauthoryear{Park, Park, Shin, and Moon}{Park
  et~al\mbox{.}}{2017}]%
        {Park2017}
\bibfield{author}{\bibinfo{person}{Sungrae Park}, \bibinfo{person}{Jun-Keon
  Park}, \bibinfo{person}{Su-Jin Shin}, {and} \bibinfo{person}{Il-Chul Moon}.}
  \bibinfo{year}{2017}\natexlab{}.
\newblock \showarticletitle{Adversarial Dropout for Supervised and
  Semi-supervised Learning}.
\newblock \bibinfo{journal}{\emph{arXiv:1707.03631}} (\bibinfo{year}{2017}).
\newblock
\urldef\tempurl%
\url{https://doi.org/abs/1707.03631}
\showDOI{\tempurl}


\bibitem[\protect\citeauthoryear{Radford, Metz, and Chintala}{Radford
  et~al\mbox{.}}{2015}]%
        {Radford2015}
\bibfield{author}{\bibinfo{person}{Alec Radford}, \bibinfo{person}{Luke Metz},
  {and} \bibinfo{person}{Soumith Chintala}.} \bibinfo{year}{2015}\natexlab{}.
\newblock \showarticletitle{Unsupervised Representation Learning with Deep
  Convolutional Generative Adversarial Networks}.
\newblock \bibinfo{journal}{\emph{arXiv:1511.06434}} (\bibinfo{year}{2015}).
\newblock
\urldef\tempurl%
\url{https://doi.org/abs/1511.06434}
\showDOI{\tempurl}


\bibitem[\protect\citeauthoryear{Rumelhart, Hinton, and Williams}{Rumelhart
  et~al\mbox{.}}{1986}]%
        {Rumelhart1986}
\bibfield{author}{\bibinfo{person}{David~E. Rumelhart},
  \bibinfo{person}{Geoffrey~E. Hinton}, {and} \bibinfo{person}{Ronald~J.
  Williams}.} \bibinfo{year}{1986}\natexlab{}.
\newblock \showarticletitle{Learning representations by back-propagating
  errors}.
\newblock \bibinfo{journal}{\emph{Nature}} \bibinfo{volume}{323},
  \bibinfo{number}{9} (\bibinfo{year}{1986}), \bibinfo{pages}{533 -- 536}.
\newblock


\bibitem[\protect\citeauthoryear{Shokri and Shmatikov}{Shokri and
  Shmatikov}{2015}]%
        {Shokri2015}
\bibfield{author}{\bibinfo{person}{Reza Shokri} {and} \bibinfo{person}{Vitaly
  Shmatikov}.} \bibinfo{year}{2015}\natexlab{}.
\newblock \showarticletitle{Privacy-Preserving Deep Learning}. In
  \bibinfo{booktitle}{\emph{Proceedings of the 22nd ACM SIGSAC Conference on
  Computer and Communications Security (CCS)}}. \bibinfo{pages}{1310--1321}.
\newblock


\bibitem[\protect\citeauthoryear{Simonyan and Zisserman}{Simonyan and
  Zisserman}{2014}]%
        {Simonyan2014}
\bibfield{author}{\bibinfo{person}{Karen Simonyan} {and}
  \bibinfo{person}{Andrew Zisserman}.} \bibinfo{year}{2014}\natexlab{}.
\newblock \showarticletitle{Very Deep Convolutional Networks for Large-Scale
  Image Recognition}.
\newblock \bibinfo{journal}{\emph{arXiv:1409.1556}} (\bibinfo{year}{2014}).
\newblock
\urldef\tempurl%
\url{https://doi.org/abs/1409.1556}
\showDOI{\tempurl}


\bibitem[\protect\citeauthoryear{Sun, Wang, Cao, Philip, Srisa-an, and
  Leow}{Sun et~al\mbox{.}}{2017}]%
        {sun2017sequential}
\bibfield{author}{\bibinfo{person}{Lichao Sun}, \bibinfo{person}{Yuqi Wang},
  \bibinfo{person}{Bokai Cao}, \bibinfo{person}{S~Yu Philip},
  \bibinfo{person}{Witawas Srisa-an}, {and} \bibinfo{person}{Alex~D Leow}.}
  \bibinfo{year}{2017}\natexlab{}.
\newblock \showarticletitle{Sequential Keystroke Behavioral Biometrics for
  Mobile User Identification via Multi-view Deep Learning}. In
  \bibinfo{booktitle}{\emph{Joint European Conference on Machine Learning and
  Knowledge Discovery in Databases (ECML-PKDD)}}. \bibinfo{pages}{228--240}.
\newblock


\bibitem[\protect\citeauthoryear{Szegedy, Liu, Jia, Sermanet, Reed, Anguelov,
  Erhan, Vanhoucke, and Rabinovich}{Szegedy et~al\mbox{.}}{2015}]%
        {Szegedy2015}
\bibfield{author}{\bibinfo{person}{Christian Szegedy}, \bibinfo{person}{Wei
  Liu}, \bibinfo{person}{Yangqing Jia}, \bibinfo{person}{Pierre Sermanet},
  \bibinfo{person}{Scott Reed}, \bibinfo{person}{Dragomir Anguelov},
  \bibinfo{person}{Dumitru Erhan}, \bibinfo{person}{Vincent Vanhoucke}, {and}
  \bibinfo{person}{Andrew Rabinovich}.} \bibinfo{year}{2015}\natexlab{}.
\newblock \showarticletitle{Going deeper with convolutions}. In
  \bibinfo{booktitle}{\emph{IEEE Conference on Computer Vision and Pattern
  Recognition (CVPR)}}. \bibinfo{pages}{1--9}.
\newblock


\bibitem[\protect\citeauthoryear{Teerapittayanon, McDanel, and
  Kung}{Teerapittayanon et~al\mbox{.}}{2017}]%
        {Teerapittayanon2017}
\bibfield{author}{\bibinfo{person}{Surat Teerapittayanon},
  \bibinfo{person}{Bradley McDanel}, {and} \bibinfo{person}{H.~T. Kung}.}
  \bibinfo{year}{2017}\natexlab{}.
\newblock \showarticletitle{Distributed Deep Neural Networks Over the Cloud,
  the Edge and End Devices}. In \bibinfo{booktitle}{\emph{IEEE 37th
  International Conference on Distributed Computing Systems (ICDCS)}}.
  \bibinfo{pages}{328--339}.
\newblock


\bibitem[\protect\citeauthoryear{Yosinski, Clune, Bengio, and Lipson}{Yosinski
  et~al\mbox{.}}{2014}]%
        {Yosinski2014}
\bibfield{author}{\bibinfo{person}{Jason Yosinski}, \bibinfo{person}{Jeff
  Clune}, \bibinfo{person}{Yoshua Bengio}, {and} \bibinfo{person}{Hod Lipson}.}
  \bibinfo{year}{2014}\natexlab{}.
\newblock \showarticletitle{How transferable are features in deep neural
  networks?}. In \bibinfo{booktitle}{\emph{Proceedings of the 27th
  International Conference on Neural Information Processing Systems (NIPS)}}.
  \bibinfo{pages}{3320--3328}.
\newblock


\bibitem[\protect\citeauthoryear{Yu, Lukefahr, Palframan, Dasika, Das, and
  Mahlke}{Yu et~al\mbox{.}}{2017}]%
        {Yu2017}
\bibfield{author}{\bibinfo{person}{Jiecao Yu}, \bibinfo{person}{Andrew
  Lukefahr}, \bibinfo{person}{David Palframan}, \bibinfo{person}{Ganesh
  Dasika}, \bibinfo{person}{Reetuparna Das}, {and} \bibinfo{person}{Scott
  Mahlke}.} \bibinfo{year}{2017}\natexlab{}.
\newblock \showarticletitle{Scalpel: Customizing DNN Pruning to the Underlying
  Hardware Parallelism}. In \bibinfo{booktitle}{\emph{Proceedings of the 44th
  Annual International Symposium on Computer Architecture (ISCA)}}.
  \bibinfo{pages}{548--560}.
\newblock


\bibitem[\protect\citeauthoryear{Zhang, Yang, and Chen}{Zhang
  et~al\mbox{.}}{2016}]%
        {Zhang2016}
\bibfield{author}{\bibinfo{person}{Qingchen Zhang},
  \bibinfo{person}{Laurence~T. Yang}, {and} \bibinfo{person}{Zhikui Chen}.}
  \bibinfo{year}{2016}\natexlab{}.
\newblock \showarticletitle{Privacy Preserving Deep Computation Model on Cloud
  for Big Data Feature Learning}.
\newblock \bibinfo{journal}{\emph{IEEE Trans. Comput.}} \bibinfo{volume}{65},
  \bibinfo{number}{5} (\bibinfo{year}{2016}), \bibinfo{pages}{1351--1362}.
\newblock


\bibitem[\protect\citeauthoryear{Zhang}{Zhang}{2004}]%
        {Zhang2004}
\bibfield{author}{\bibinfo{person}{Tong Zhang}.}
  \bibinfo{year}{2004}\natexlab{}.
\newblock \showarticletitle{Solving Large Scale Linear Prediction Problems
  Using Stochastic Gradient Descent Algorithms}. In
  \bibinfo{booktitle}{\emph{Proceedings of the 21st International Conference on
  Machine Learning (ICML)}}. \bibinfo{pages}{116--123}.
\newblock


\end{thebibliography}

\end{document}